\documentclass{article} 
\usepackage{iclr2025_conf_mod,times}

\usepackage{preamble}

\title{Stochastic Sampling from Deterministic Flow Models}

\author{%
  Saurabh Singh \\
  Google DeepMind\\
  \texttt{saurabhsingh@google.com} \\
   \And
   Ian Fischer \\
   Google DeepMind \\
   \texttt{iansf@google.com} \\
}

\iclrfinalcopy 
\begin{document}

\maketitle

\begin{abstract}
Deterministic flow models, such as rectified flows, offer a general framework for learning a deterministic transport map between two distributions, realized as the vector field for an ordinary differential equation (ODE).
However, they are sensitive to model estimation and discretization errors and do not permit different samples conditioned on an intermediate state, limiting their application.
We present a general method to turn the underlying ODE of such flow models into a family of stochastic differential equations (SDEs) that have the same marginal distributions.
This method permits us to derive families of \emph{stochastic samplers}, for fixed (e.g., previously trained) \emph{deterministic} flow models, that continuously span the spectrum of deterministic and stochastic sampling, given access to the flow field and the score function.
Our method provides additional degrees of freedom that help alleviate the issues with the deterministic samplers and empirically outperforms them.
We empirically demonstrate advantages of our method on a toy Gaussian setup and on the large scale ImageNet generation task.
Further, our family of stochastic samplers provide an additional knob for controlling the diversity of generation, which we qualitatively demonstrate in our experiments.
\end{abstract}

\section{Introduction}

Deterministic flow models, including Rectified Flow~\citep{liu2022flow}, Flow Matching~\citep{lipman2022flow,tong2023improving}, Stochastic Interpolants~\citep{albergo2022building, albergo2023stochastic}, and probability flow ODE~\citep{song2020score} learn a reversible deterministic transport between two end distributions $p_0(x_0)$ and $p_1(x_1)$.
Diffusion models require one of the distributions to be a Gaussian distribution, though generalizations exist~\citep{yoon2024score}.
In contrast, Rectified Flows, Stochastic Interpolants, and Flow Matching offer a general framework for learning deterministic transports, without this restriction.
While deterministic transport enables efficient deterministic sampling, e.g. by the rectification procedure suggested by~\citet{liu2022flow}, stochastic sampling may be desirable for:
(1) robustness to estimation errors in the learned flow model,
(2) ability to produce random samples conditioned on an intermediate state $x_t, t \in [0, 1]$,
and (3) robustness to discretization error resulting from discrete step sampling from a continuous time model.
We present a new theorem (\cref{thm:general}) that provides a recipe to create an infinite family of parameterized stochastic samplers, given access to the flow field and the score function for the marginal distributions.
Our result provides a general and unified view, while including a few existing proposals (e.g. in \cite{huang2021variational,berner2022optimal}) as special cases.

The deterministic transport specifies a deterministic mapping between the samples from the two distributions and is realized as a learned vector field corresponding to an ordinary differential equation (ODE).
However, if one distribution is chosen to be a Gaussian, these Flow models can be viewed as reparameterizations of other deterministic models that also choose a Gaussian as one of the distributions e.g. probability flow ODEs arising from Gaussian diffusion models.
We refer to such models as \emph{Gaussian flow} models.
Transport map learning algorithms such as Gaussian flows are practical to train and enable applications like generative modeling~\citep{ramesh2022hierarchical,lu2022dpmsolver,saharia2022photorealistic,esser2024scaling}, stylization~\citep{isola2017image,meng2022sdedit}, and image restoration~\citep{delbracio2023inversion,rombach2022high,lugmayr2022repaint,kawar2022denoising}, to name a few.
However, corresponding deterministic sampler has limitations that we empirically demonstrate on a toy Gaussian task, where it exhibits a bias and consistently underestimates the variance of the target distribution, as seen in \cref{fig:toy_samplers}.
To enable stochastic sampling from such deterministic models, we provide a special case of our general result to turn the underlying ODE of Gaussian flow models into a family of stochastic differential equations (SDEs) that have the same marginal distributions.
Our stochastic samplers allow trading the bias of deterministic sampler for increased variance in the estimated mean and variance parameters (\cref{fig:toy_bias_var}).
Since, our method requires access to the score function for the marginal distributions, we impute it directly from the given flow model, alleviating the need for learning it separately.
This method permits us to derive families of \emph{stochastic samplers}, for fixed (e.g., previously trained) \emph{deterministic} Gaussian flow models, that allow flexible and time dependent injection of stochasticity during sampling, enabling both deterministic and stochastic sampling.
This additional degree of freedom allows exploration of stochastic samplers that can help alleviate the issues with the deterministic samplers and outperform them.
We demonstrate this empirically on a toy Gaussian setup, as well as on the large scale ImageNet generation task.
The stochastic samplers also provide an additional knob for controlling the diversity of generation as we qualitatively demonstrate in our experiments, and are compatible with classifier-free guidance~\citep{ho2022classifier}, as can be seen in \cref{fig:imagenet_cfg,fig:qual_alpha_vs_lambda}.

\textbf{Our key contributions are:}
\textbf{(1)} Specification of a flexible family of SDEs (\cref{thm:general}) that have the same marginal distributions as a given SDE or a flow model, enabling exploration of sampling schemes for a given fixed model,
\textbf{(2)} Derivation of new as well as existing special cases directly from \cref{thm:general} (\cref{cor:diffusion_corollary} and \cref{cor:detflow_corollary}) demonstrating generality of \cref{thm:general},
\textbf{(3)} Study of a set of SDE families corresponding to Gaussian flow models, derived using \cref{thm:general}, on both a toy as well as a large scale ImageNet setup, demonstrating flexible stochastic sampling and controllable diversity in generation, \emph{without requiring retraining} (\cref{tab:sde_family_table}, \cref{fig:imagenet_cfg,fig:qual_alpha_vs_lambda}).

\section{Background}
\label{sec:background}

\begin{figure}[t]
\centering
\includegraphics[width=0.99\textwidth]{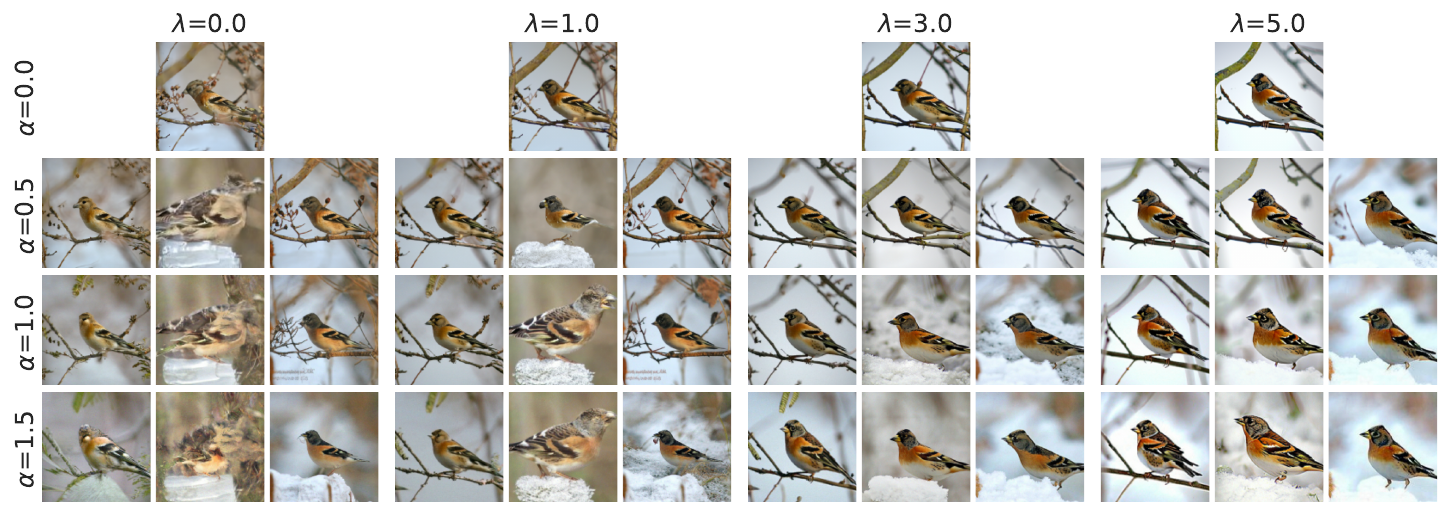}
\caption{%
    \textbf{Stochastic sampling improves diversity at all classifier-free guidance levels.}
    We visualize samples from a rectified flow model at four classifier-free guidance levels $\lambda$ (\cref{sec:score_function}) and at four stochasticity scales $\alpha$ for NonSingular (\cref{tab:sde_family_table}). 
    Three samples are shown for each configuration where the sampling starts at the same draw from $p_1(x_1)$.
    When $\alpha=0$, the sampler is deterministic and samples are the same (therefore we show only one).
    When $\lambda=0$, there is no classifier-free guidance.
    Note the increased diversity as $\alpha$ increases. More examples in \cref{fig:qual_alpha_vs_lambda}.
}
\label{fig:imagenet_cfg}
\vspace{-0.4cm}
\end{figure}

\paragraph{Notation.}
Throughout this work we use small Latin letters $t, x, y$ etc. to represent scalar and vector variables, $f, g$ etc. to represent functions, Greek letters $\alpha, \beta$ etc. to represent (hyper-)parameters, and capital letters $G$ to represent matrices.
With a slight abuse of notation we use lower case letters $x$ to represent both the random variable and a particular value $x \sim p(x)$.
Whenever unambiguous, we suppress the dependence of state $x_t$ on time $t$ as $x \equiv x_t$, and dependence of functions on state $x_t$ and time $t$ as $f \equiv f(x_t, t)$ to simplify notation.

We briefly discuss rectified flow and continuous time diffusion models.
Refer to \cite{liu2022flow, song2020score} for details.

\subsection{Rectified Flow}
\label{sec:rect_flow}

Let $x_0 \sim p_0(x_0) \in \R^d$ be the draws from the data distribution $p_0$ that we are interested in learning and sampling from.
Let $x_1\sim p_1(x_1) \in \R^d$ be an easy to sample source distribution.
Loosely, the key idea behind the diffusion and flow family of models is to learn a mapping that either stochastically or deterministically transforms a sample from $p_1$, in an iterative manner, to produce a sample from $p_0$.
Let $\nu(x_0, x_1)$ be an arbitrary coupling distribution for the two random variables $x_0, x_1$ such that $p_0(x_0) = \int \nu(x_0, x_1)dx_1, p_1(x_1) = \int \nu(x_0, x_1)dx_0$.
A simple choice is the product of the two: $\nu(x_0, x_1) \equiv p_0(x_0)p_1(x_1)$.
To construct a rectified flow first an interpolation between the two variables is defined as $x_t \equiv h(x_0, x_1, t)$ that is differentiable w.r.t. time.
The default interpolation proposed and studied in \citet{liu2022flow} is:
\begin{equation}
\label{eq:rect_flow_default_interp}
x_t = (1-t)x_0 + t x_1, \quad t \in [0, 1].
\end{equation}
With the above, rectified flow learns a vector field $v(x_t, t)$ by minimizing the following objective:
\begin{equation}
\label{eq:rectflow_obj}
v(x, t) = \argmin_{v'} \mathbb E_{(x_0, x_1)\sim \nu} \left[ \int_0^1 \left\lVert \frac{dx_t}{dt} - v'(x_t, t) \right\rVert^2 dt \right].
\end{equation}
The solution to the above optimization problem is $v(x, t) \equiv \mathbb E[x_1 - x_0 |x,t]$ and is referred to as 1-Rectified flow.
Since $v(x, t)$ is not available in closed-form in general, $v$ is typically parameterized with parameters $\theta$ and optimization in \Cref{eq:rectflow_obj} is performed w.r.t. $\theta$.
In the rest of the paper, we drop this dependence on the parameters in notation as we assume a model $v(x, t)$ to be given.
Note that a closed-form expression is available when $p_0, p_1$ are Gaussian (see \cref{sec:closed_gaussian_app}).
We use this expression for the toy setup in our experiments. For example, the biased deterministic sampler in \cref{fig:toy_samplers} is using the ground truth flow field.
Once the flow $v(x_t, t)$ is estimated, samples from $p_0(x_0)$ can be produced by drawing a sample from $p_1(x_1)$ and simulating the flow backward in time, using:
\begin{align}
\label{eq:rect_flow_ode}
dx = v(x, t)dt
\end{align}

\subsection{Score based diffusion with stochastic differential equations}
\label{sec:diffusion}

The general idea in this family of methods is to specify a forward stochastic process that slowly transforms the data density $p_0(x_0)$ into an easy to sample source density $p_1(x_1)$.
\citet{song2020score} specified such a process using an It\^o SDE of the following form:
\begin{align}
\label{eq:general_sde_paper}
dx = f(x, t) dt + G(x, t)dW_t
\end{align}
where $f(x,t): \R^d\times [0, 1] \rightarrow \R^d$ is the drift coefficient, $G(x, t) : \R^d\times [0, 1] \rightarrow \R^d \times \R^d$ is state and time dependent diffusion coefficient and $W_t$ is the Wiener process.
Choosing\footnote{%
    \citet{song2020score} provide general results for $G(x, t)$ which we omit here for brevity.
} $G \equiv g(t): [0, 1] \rightarrow \R$ and using results from \cite{anderson1982reverse}, a reverse time SDE can be specified that has the same marginals as \cref{eq:general_sde_paper}:
\begin{align}
\label{eq:reverse_sde_song}
dx = [f(x, t) - g^2(t)\nabla_{x} \ln p_t(x)]dt + g(t)d\tilde W_t
\end{align}
where $\tilde W_t$ is a standard Wiener process with time running backwards.
Note that the time reversal requires access to the score function $\nabla_{x} \ln p_t(x)$.
Score matching \citep{vincent2011connection} can be used to learn an estimate for the score for all $t$ \citep{song2020score}, which can then be used to simulate reverse time dynamics starting with a sample from $p_1(x_1)$ to produce a sample from $p_0(x_0)$ at $t=0$.
A forward deterministic process can also be derived from the above that has the same marginal densities $p_t(x)$:
\begin{align}
\label{eq:probflow_ode_song}
dx = \left[f(x, t) - \frac{1}{2}g^2(t)\nabla_{x} \ln p_t(x)\right]dt
\end{align}
The above ODE is also referred to as the probability flow ODE.
Samples can be generated using the above ODE in a similar fashion as rectified flow, by simulating the ODE backwards in time.

\section{Deriving stochastic samplers}
\label{sec:sampling}

\begin{figure}[t]
\centering
\includegraphics[width=0.99\textwidth]{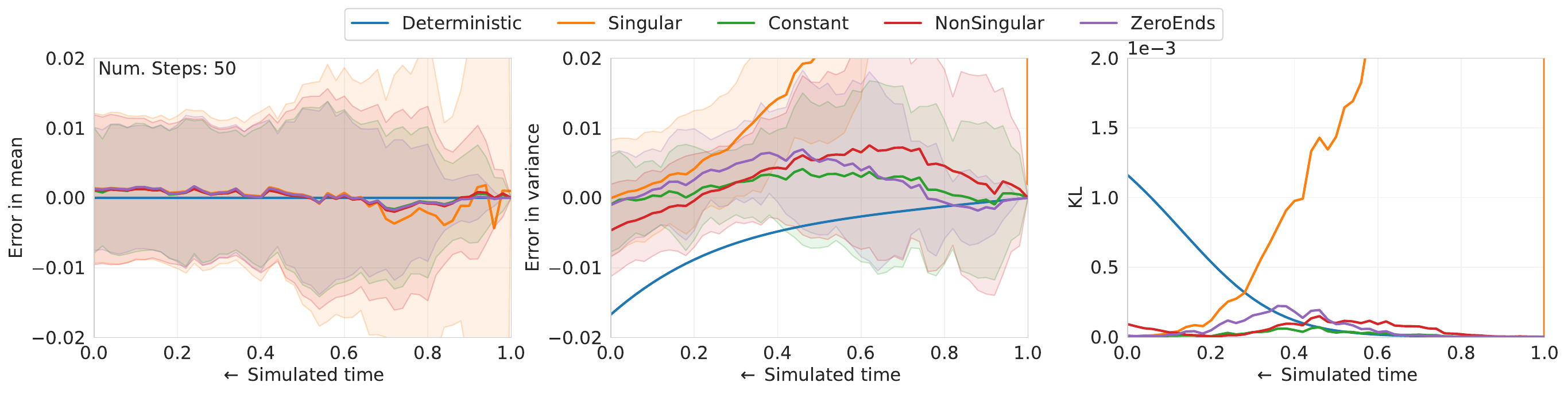}
\caption{%
    \textbf{Discretization of deterministic flow leads to bias.}
    Comparison of samplers from \cref{tab:sde_family_table} on the two Gaussian toy problem (\cref{sec:toy_details}).
    Deterministic underestimates the variance parameter, but the stochastic samplers avoid that issue, in exchange for variance in the parameter estimation.
    Singular's variance diverges if we start from $t=1$, so instead we start the sampler at $t=1-10^{-3}$, which allows it to eventually converge by $t=0$.
}
\label{fig:toy_samplers}
\end{figure}

\paragraph{Method intuition.}
Probability flow ODEs \citep{song2020score}, proposed in the context of diffusion models, provide a deterministic sampling method for diffusion models.
These ODEs have the same marginal distribution $p_t(x)$ at all $t$ as the original SDE from which they are derived.
Here, we take the reverse path: we start from an ODE (corresponding to the Gaussian flow model) and deduce the family of SDEs that have the same marginal distributions at all $t$ as the original ODE.
Before we introduce the general result, we will show a naive approach that gives an SDE with a problematic singularity, motivating the need for the generalization.

\subsection{A singular SDE corresponding to Gaussian flow}
\label{sec:gaussian_flow_sde}

For Gaussian flow, $p_1(x_1)\equiv N(x_1; \mu_1, \sigma_1^2I)$ is assumed to be Gaussian.
With interpolation $x_t = (1-t)x_0 + tx_1$, the perturbation kernel $p(x_t | x_0) = N(x_t; (1-t)x_0+t\mu_1, t^2\sigma_1^2I)$ is also Gaussian.

Note that since $x_0, x_1$ are independent, we can directly infer the first and second moments $\mu_t, \Sigma_t$ for the marginals $p_t(x)$ as $\mu_t =  (1-t)\mu_0 + t\mu_1$ and $\Sigma_t = (1-t)^2\Sigma_0 + t^2\sigma_1^2I$.
With these constraints and choosing $\mu_1 \equiv 0, \sigma_1 \equiv 1$, we can solve for drift and diffusion coefficients that yield the same marginal distributions:
\begin{align}
\label{eq:singular_choices}
f(x, t) &= -\frac{x}{1-t} & g(t) &= \sqrt{\frac{2t}{1-t}}
\end{align}
See \cref{sec:singular_sde_app} for the details and a more general expression for arbitrary $\mu_1, \sigma_1$.
The coefficients $f(x, t), g(t)$ are singular at the boundary $t=1$ of the interval.
Consequently, simulation methods such as Euler-Maruyama, that need $f(x, t), g(t)$ to be Lipschitz are not guaranteed to work at the boundary (see \cref{fig:toy_samplers} and \cref{sec:guass_experiments}).
We refer to this SDE as the Singular SDE.
An empirical trick that is often used in such cases is to assume $p_{1-\epsilon}(x_{1-\epsilon}) \approx p_1(x_1),  \epsilon \ll 1$.
However, this can lead to unpredictable behavior and we show how to avoid it in the following section.

\subsection{Set of SDEs that share the same marginal distribution $p_t(x)$}
\label{sec:general_sde}

First we state our general result with the diffusion coefficient $G(x, t)$ a function of both the state $x$ and time $t$, and then state simpler forms more directly applicable to models used in practice.

\begin{restatable}{theorem}{mainresult}
\label{thm:general}
Let $p_t(x)$ be the probability density of the solutions of the SDE in \cref{eq:general_sde_paper}
evolving as $\frac{\partial p_t}{\partial t}$.
Then, the probability density of solutions of the following set of SDEs, indexed by $\GT, \gamma_t$, also evolves as $\frac{\partial p_t}{\partial t}$.
\begin{align}
\label{eq:result_sde}
dx = \bar f(x,t)dt + \bar G(x, t)dW_t
\end{align}
where
\begin{align}
\bar f &= f - \frac{1}{2}\left(\nabla \cdot [(1-\gamma_t)GG^T-\GT\GT^T] + [(1-\gamma_t)GG^T-\GT\GT^T]\cdot\nabla\ln p_t \right)\\
\bar G &= [\gamma_t GG^T + \GT\GT^T]^{\frac{1}{2}}
\end{align}
and $\GT \equiv \GT(x, t), \gamma_t \ge 0$ are arbitrary functions such that the solutions of \cref{eq:result_sde} exist and are unique.
\end{restatable}

Proof of \cref{thm:general} is given in \cref{sec:mainresult_proof_app} and follows from manipulations of Fokker-Planck-Kolmogorov (FPK) equations corresponding to \cref{eq:result_sde}.

\cref{thm:general} implies that given the same initial density $p_0(x)$, evolution according to both \cref{eq:general_sde_paper} and \cref{eq:result_sde} will have the same marginal densities $p_t(x)$ for all times $t$.
Further, \cref{eq:result_sde} can be simulated backward in time using the result from \citet{anderson1982reverse}, again with the same marginal densities $p_t(x)$.
Consequently, \cref{eq:general_sde_paper} can be simulated forward or backward in time using any member of the family specified by \cref{eq:result_sde}.
Note that setting $\gamma_t=1, \GT = 0$ recovers the original SDE in \cref{eq:general_sde_paper}, while setting $\gamma_t=0,\GT = 0$ recovers the general probability flow ODE from \citet[eq. 37]{song2020score}.
Additionally, $\GT$ is particularly useful for deterministic flow models, further discussed in \cref{cor:detflow_corollary}.
\cref{thm:general} gives a recipe for developing particular samplers, such as those in the remainder of this section, some of which have appeared in the literature.
A priori, \cref{thm:general} cannot determine which concrete sampler will be best for a given application, but since the samplers do not require any training to use, it is possible to postpone the choice of sampler to an empirical analysis at test time.

The flexibility afforded by \cref{eq:result_sde} is particularly useful
\textbf{(1)} in the presence of singularities in the drift and diffusion coefficients $f$ and $G$ respectively of \cref{eq:general_sde_paper},
\textbf{(2)} in the presence of errors resulting from finite discretization,
and \textbf{(3)} for flexible specification of the diffusion coefficient in generative applications.
Our experimental evaluations primarily focus on these aspects of \cref{thm:general}.

A direct consequence of \cref{thm:general}, by defining $\GT\equiv0, G \equiv g(t)I$, is the following corollary applicable to commonly used generative diffusion models with additive noise:
\begin{restatable}{corollary}{corollaryone}
\label{cor:diffusion_corollary}
For the SDE in \cref{eq:general_sde_paper} with $G\equiv g(t)I$, a subset of SDEs prescribed by \cref{thm:general} and indexed by $\gamma_t$ is:
\begin{align}
dx &= \left[f(x, t) - \frac{(1 - \gamma(t))g^2(t)}{2}\nabla_x \ln p_t(x)\right]dt + \sqrt{\gamma(t)} g(t) dW_t
\label{eq:diffusion_corollary}
\end{align}
\end{restatable}
Proof in \cref{sec:corollaryone_proof_app}.
Note that choosing $\gamma_t = 0$ results in the probability flow ODE specified in \cref{eq:probflow_ode_song}.
Intuitively, the members in the family differ in terms of the amount of noise injected as a function of time.
$\gamma_t = 0$ yields a purely deterministic simulation; $\gamma_t > 0$ yields a variety of stochastic simulations.
Further, similar special cases discussed in \cite{huang2021variational} and \cite{berner2022optimal} also directly follow from \cref{thm:general} as well.

Some properties of \cref{cor:diffusion_corollary}:
\begin{enumerate}[leftmargin=*]
    \item $\gamma(t)$ can be chosen at sampling time and doesn't affect the training of the score function.
    \item With $\gamma(t) = \hat \gamma^2(t)g^{-2}(t)$, where $\hat \gamma(t)$ is an arbitrary function (satisfying constraints of \cref{thm:general}), we can choose an arbitrary diffusion term at sampling time.
    For example, choosing $\gamma_t = \gamma^2/g^2(t)$ leads to a constant diffusion coefficient.
    \item For the SDE specified by \cref{eq:singular_choices}, we can choose $\gamma(t) = (1-t)\hat \gamma^2(t)g^{-2}(t)$ to get rid of the singularity in the diffusion term.
\end{enumerate}

Note that \cref{thm:general} can be used whenever we have access to the score function $\nabla_x \ln p_t$.
Next, we first construct a specialized solution based on \cref{thm:general} for deterministic flow models that enables flexible control of both drift and diffusion coefficients, and apply it to the special case of deterministic Gaussian flows where the score function can be imputed from the velocity field (\cref{sec:score_function}).
Recall that deterministic flows specify a transport via the ODE $dx = v(x, t)dt$.
This ODE can be viewed as an SDE where the diffusion term has been set to zero.
Choosing $G \equiv 0, \GT \equiv \tilde g(t)I$ in \cref{thm:general} gives \cref{cor:detflow_corollary}, which enables deriving stochastic samplers for Gaussian flow models:
\begin{restatable}{corollary}{corollarytwo}
\label{cor:detflow_corollary}
For the ODE in \cref{eq:rect_flow_ode}, a subset of SDEs prescribed by \cref{thm:general} and indexed by $\tilde g(t)$ is
\begin{align}
\label{eq:detflow_corollary}
dx &= \left[v(x, t) + \frac{\tilde g^2(t)}{2}\nabla_x \ln p_t(x)\right]dt + \tilde g(t) dW_t
\end{align}
\end{restatable}
Proof in \cref{sec:corollarytwo_proof_app}.
\cref{cor:detflow_corollary} enables flexible specification of a time dependent diffusion coefficient $\tilde g(t)$, allowing the introduction of stochasticity in the simulation of otherwise deterministic models, \emph{purely at sampling time}.
Note that \cref{eq:detflow_corollary} requires access to the score function $\nabla_x \ln p_t(x)$ for the marginal distributions $p_t(x)$.  In \cref{sec:score_function}, we describe how the score function can be imputed from the learned flow model $v(x, t)$ for the special case of Gaussian flow models.
It can be verified that the particular choice of $f$ and $g$ in \cref{eq:singular_choices}  satisfy \cref{eq:detflow_corollary} by using the expression for the score from \cref{eq:rect_score}.

\begin{table}[t]
  \caption{%
    Examples of SDEs that have the same marginal distribution $p_t(x)$ as a given Gaussian flow specified by $v\equiv v(x,t)$.
    $\alpha \ge 0$ is a scale parameter that varies the magnitude of the diffusion coefficient $g$.
    Each of these behaves differently when discretized and simulated (\Cref{fig:toy_samplers} and \cref{sec:diff_coeff_mag_qual}).
    These and infinitely many more can be constructed using the scheme in \cref{eq:detflow_corollary}.
  }
  \label{tab:sde_family_table}
  \centering
  \begin{tabular}{lLLl}
    \toprule
    Name &\tilde g(t) & \tilde f(x, t) &  Description  \\
    \midrule
    Deterministic &  0 & v & Base flow model  \\
    Constant & \alpha & v + \frac{\alpha^2}{2} \nabla_x \ln p_t & \text{Constant $g$, singular $f$} \\
    Singular & \alpha \sqrt{t/(1-t)} & v + \frac{\alpha^2}{2}\frac{t}{1-t} \nabla_x \ln p_t  & Singular $g, f$\\
    NonSingular & \alpha\sqrt{t}, & v + \frac{\alpha^2}{2}t\nabla_x \ln p_t & \text{Non-singular} $g,f$ \\
    ZeroEnds & \alpha\sqrt{t(1-t)}, & v + \frac{\alpha^2}{2}t(1-t)\nabla_x \ln p_t & Non-singular $g,f, g(0)=g(1)=0$ \\
    \bottomrule
  \end{tabular}
\end{table}

While infinitely many choices are available for $\tilde g$, we consider a few interesting ones listed in the \cref{tab:sde_family_table}, constructed by choosing integer powers of $t$ and $1-t$ and introducing a scaling coefficient $\alpha$, for experimental evaluations.
Note that the only degree of freedom in \cref{tab:sde_family_table} is the choice of $\tilde g(t)$, which determines $\tilde f(x, t)$, given the flow field $v(x, t)$ and the score $\nabla_{x}\ln p_t(x)$.
The $f(x,t)$ is singular in Constant because the score $\nabla_x \ln p_t(x_t)$, as computed in \cref{eq:rect_score}, has $t$ in the denominator, making $f(x, t)$ singular at $t=0$.
The choice in NonSingular precisely eliminates this singularity.
\Cref{fig:toy_samplers} compares these choices in a toy setup; \cref{sec:experiments} has comparisons on ImageNet.

\subsection{Score function and classifier free guidance for a Gaussian flow model}
\label{sec:score_function}

Recall that \cref{thm:general} requires access to the score function.
For Gaussian flows, the score function can be inferred from the velocity field itself, alleviating the need to learn it separately.
This result is known (see e.g. \cite{zheng2023guided} in the context of flow matching) and we present it here in our setting.
For Gaussian flows, with $p_1(x_1) \equiv N(x_1; \mu_1, \sigma_1^2I)$ and interpolation specified in \cref{eq:rect_flow_default_interp}, the score can be computed as:
\begin{align}
\label{eq:rect_score}
\nabla_{x} \ln p_t(x) &= \frac{-(1-t)v(x, t)+\mu_1-x}{t\sigma_1^2}
\end{align}
where $v(x, t) = \E [x_1-x_0 | x, t]$ is the estimated flow.
Proof is provided in \cref{sec:score_fn_app}.
Note that the score function can also be estimated given $\E [x_0 | x, t]$ or $\E [x_1 | x, t]$.
In summary, the expression follows directly from using results from Denoising Score Matching \citep{vincent2011connection} and the Gaussianity of $p_1(x_1)$.
Similar expressions can be derived for other interpolations that are linear in $x_0, x_1$.
With access to the score function and linearity of \cref{eq:rect_score} in $v$ we can define classifier free guided~\citep{ho2022classifier} Gaussian flow as:
\begin{align}
v_{\text{cfg}}(x, t, c) = (1 + \lambda)v(x, t, c) - \lambda (v(x, t, c=\varnothing)
\end{align}
where $c$ indicates extra conditioning information as in classifier free guidance, $\varnothing$ indicates no conditioning and $\lambda$ specifies the relative strength of the guidance.
$\lambda = 0$ reduces to class conditional sampling, while $\lambda > 0$ puts a larger weight on the conditioning, biasing the sample towards the modes of the conditional distribution.
Using classifier-free guidance with a stochastic sampler will, of course, give diversity that isn't possible with a deterministic sampler, as can be seen in \cref{fig:imagenet_cfg}.
Note that \cite{xie2024reflected,dao2023flow,zheng2023guided} also consider related definitions in the context of flow matching.

\section{Experiments}
\label{sec:experiments}

Our method allows us to identify a family of SDEs that correspond to a given deterministic Gaussian flow model, enabling construction of stochastic samplers with flexible diffusion coefficients.
In our experiments we compare various samplers derived from the corresponding SDEs in \cref{tab:sde_family_table}, using Euler-Maruyama, for a given Gaussian flow model \emph{without any additional training}. 

\begin{figure}[tb]
\centering
\includegraphics[width=0.99\textwidth]{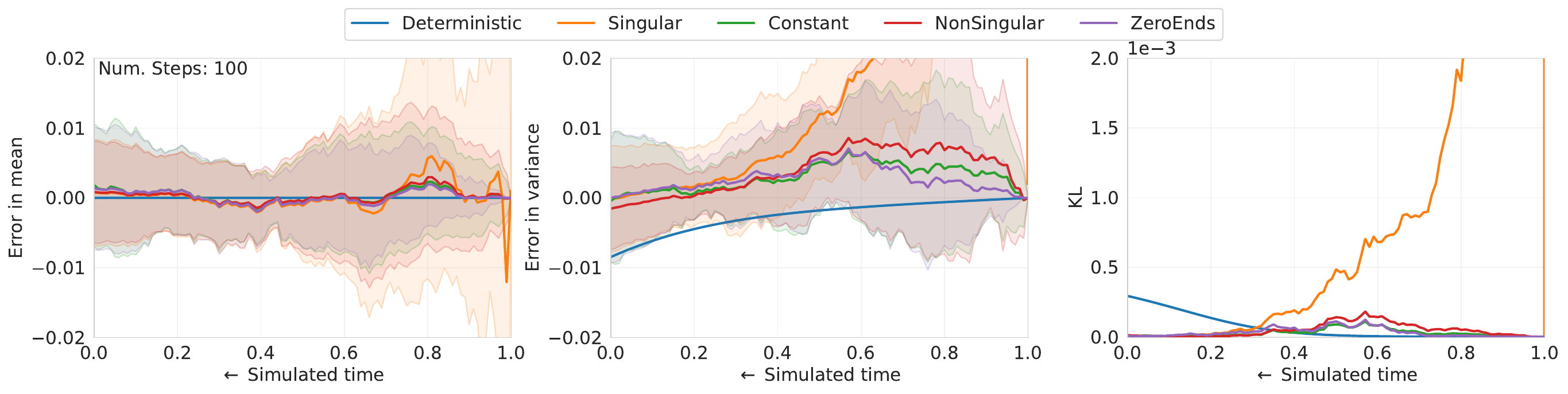}
\includegraphics[width=0.99\textwidth]{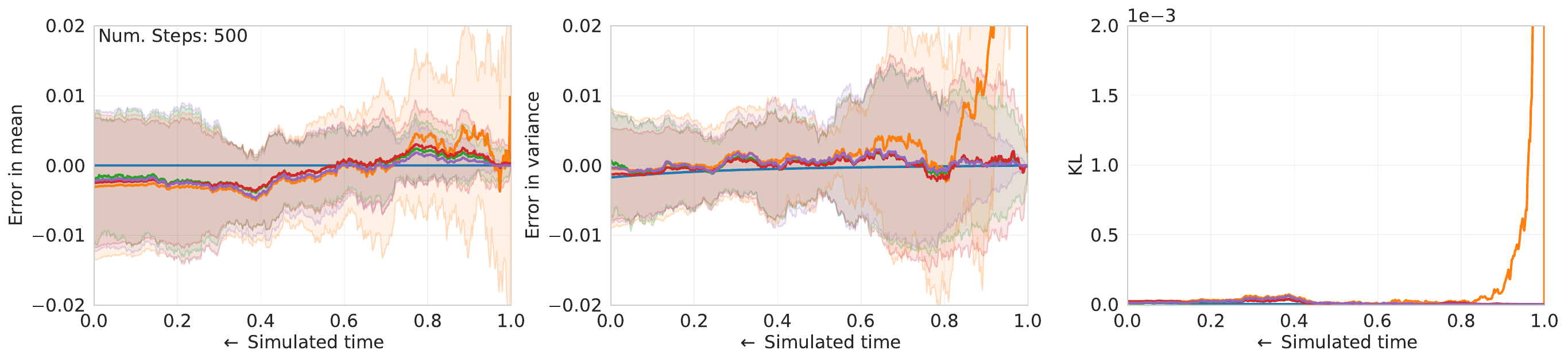}
\caption{%
    \textbf{Stochasticity is most helpful at coarser discretizations.}
    We visualize the effect of coarseness of discretization by sampling for 100 and 500 sampling steps.
    See \cref{fig:toy_samplers} for the same plots at 50 steps, which shows more extreme bias in variance for Deterministic and Singular.
}
\label{fig:toy_steps_eval}
\vspace{-0.3cm}
\end{figure}

\begin{figure}[tb]
\includegraphics[width=0.99\textwidth]{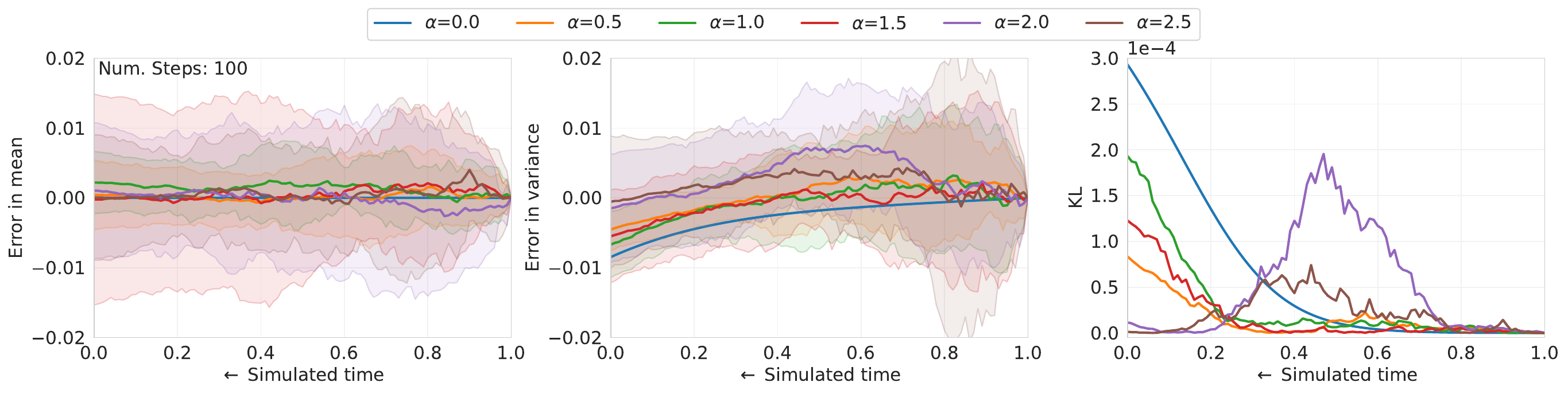}
\caption{%
    \textbf{Stochasticity helps mitigate bias.}
    We plot the error in mean and error in variance for NonSingular for a set of diffusion coefficient scales $\alpha \in \{0.0, 0.5, 1.0, 1.5, 2.0, 2.5\}$. 
    Estimates for variance at $t=0$ improve as $\alpha$ increases, leading to a drop in KL divergence from the true distribution.
    However, with very high $\alpha$ values intermediate marginals develop a bias. 
}
\label{fig:toy_bias_var}
\vspace{-0.4cm}
\end{figure}

\subsection{Comparison on estimating marginal statistics for a two Gaussian toy problem}
\label{sec:guass_experiments}

We start by considering a toy problem where both $p_0$ and $p_1$ are Gaussian.
See \cref{sec:toy_details} for details of the experimental setup and \cref{sec:code} for a JAX~\citep{jax2018github} implementation of NonSingular.

\paragraph{Discretization of deterministic flow leads to bias.}
In \cref{fig:toy_samplers}, with 50 sampling steps, we observe that the estimate for the mean is fairly accurate for all samplers for the entirety of the interval $t \in [0, 1]$.
However, the samplers differ in their behavior for variance.
Deterministic exhibits a noticeable bias and underestimates the variance (with zero variance in its estimate), with the worst estimate at $t=0$.
Stochastic samplers provide noticeably better estimates at $t=0$, but with increased variance.

\paragraph{Stochasticity is most helpful at coarser discretizations.}
In \cref{fig:toy_steps_eval} we study the effect of the number of discretization steps on the different samplers (also see \cref{fig:toy_samplers} for 50 steps).
While mean estimates are accurate for all methods, Deterministic gets increasingly biased for variance estimates as the number of sampling steps is decreased.
Stochastic samplers perform consistently well at various discretization levels for $t=0$, with significantly better estimates for fewer sampling steps.
Note that Singular has very large bias as well as variance closer to $t=1$; those improve with finer discretization.
Since Constant also has a singularity, but only in the drift term $f$, we conclude that the instability is primarily due to the singularity in Singular's diffusion term.

\paragraph{Stochasticity helps mitigate bias.}
In \cref{fig:toy_bias_var} we study the effect of diffusion coefficient scale $\alpha$ on the NonSingular sampler at 100 sampling steps.
Finite discretization introduces a bias in the deterministic sampler (when $\alpha=0$), where the variance is consistently underestimated and is worst at $t=0$.
Increased stochasticity with increasing diffusion coefficient scale ($\alpha>0$) helps mitigate this bias at the cost of increased variance.
This can be seen in the figure with larger $\alpha$ values yielding better estimate of the variance, although with larger variance in the estimate.

\subsection{Comparison of SDEs for rectified flows on ImageNet generation}
\label{sec:imagenet_experiments}

We compare the behavior of various SDEs on the sampling quality for large scale image generation using the ImageNet (2012) dataset~\citep{deng2009imagenet,ILSVRC15}.
We train rectified flow models at two different image resolutions ($64 \times 64$ and $128 \times 128$) and compare their sample quality using the Frechet Inception Distance (FID) metric~\citep{heusel2017gans} for class conditional samples.
See \cref{sec:imagenet_training_details_app} for setup details.
The results show that small but statistically significant differences exist between samplers even for metrics like FID, but the optimal sampler is likely to be application and model specific.

\paragraph{Stochasticity can improve FID.}
In \cref{tab:best_fid} we report the best FID using each SDE in \cref{tab:sde_family_table} for two image resolutions using 300 sampling steps, along with the corresponding diffusion term scale $\alpha$ and one standard deviation confidence interval.
Two key observations stand out:
\textbf{(1)} stochastic samplers tend to produce better FIDs,
and \textbf{(2)} the two non-singular samplers have much better FIDs than Deterministic or Singular.
Note that observation \textbf{(1)} has also been made previously for probability flow ODEs \citep{song2020score}.
The addition of a parameter $\alpha$ to control the strength of the stochasticity while keeping the marginal distribution $p_t$ unchanged (\cref{thm:general}), permits principled post-training optimization of the metrics like FID.

\paragraph{Non-singular samplers work well over a broad range of $\alpha$.}
In \cref{fig:imagenet_fid_alpha,fig:imagenet_fid_alpha_uncropped} we show how the FID varies with $\alpha$ for each sampler for two different image resolution models.
NonSingular and ZeroEnds attain better FID in general and are better behaved over a much larger range of the diffusion coefficient scale $\alpha$ at both resolutions. 
These samplers both have small diffusion coefficients $g(t)$ close to $t=0$; their performance indicates that noise near $t=0$ is particularly harmful.
The low variance of ZeroEnds in comparison to NonSingular indicates that a large diffusion coefficient near $t=1$ tends to introduce variance in the final FID.

\begin{table}[t]
  \caption{%
    \textbf{Stochasticity can improve FID.}
    Comparison of various samplers at their best $\alpha$ values with 300 sampling steps for ImageNet image generation task at two resolutions.
  }
  \label{tab:best_fid}
  \centering
  \begin{tabular}{lLLLL}
    \toprule
    & \multicolumn{2}{c}{$64 \times 64$} & \multicolumn{2}{c}{$128 \times 128$} \\
    \cmidrule(r){2-3}
    \cmidrule(r){4-5}
    Sampler & \text{FID} &  \alpha & \text{FID}  &  \alpha  \\
    \midrule
    Deterministic & 3.07 \pm 0.01 & 0.0  & 5.19 \pm 0.02 & 0.0 \\
    Singular & 3.07 \pm 0.01 & 0.08 & 5.13 \pm 0.04 & 0.14 \\
    Constant $g$ & 2.97 \pm 0.04 & 0.08 & 5.17 \pm 0.05 & 0.1 \\
    NonSingular & \bf{2.95 \pm 0.01} & 0.56 & \bf{4.93 \pm 0.06} & 0.42 \\
    ZeroEnds & \bf{2.95 \pm 0.01} & 0.54 & 5.03 \pm 0.01 & 0.52 \\
    \bottomrule
  \end{tabular}
\end{table}

\begin{figure}[tbp]
\centering
\begin{subfigure}[b]{0.48\linewidth}
\includegraphics[width=\linewidth]{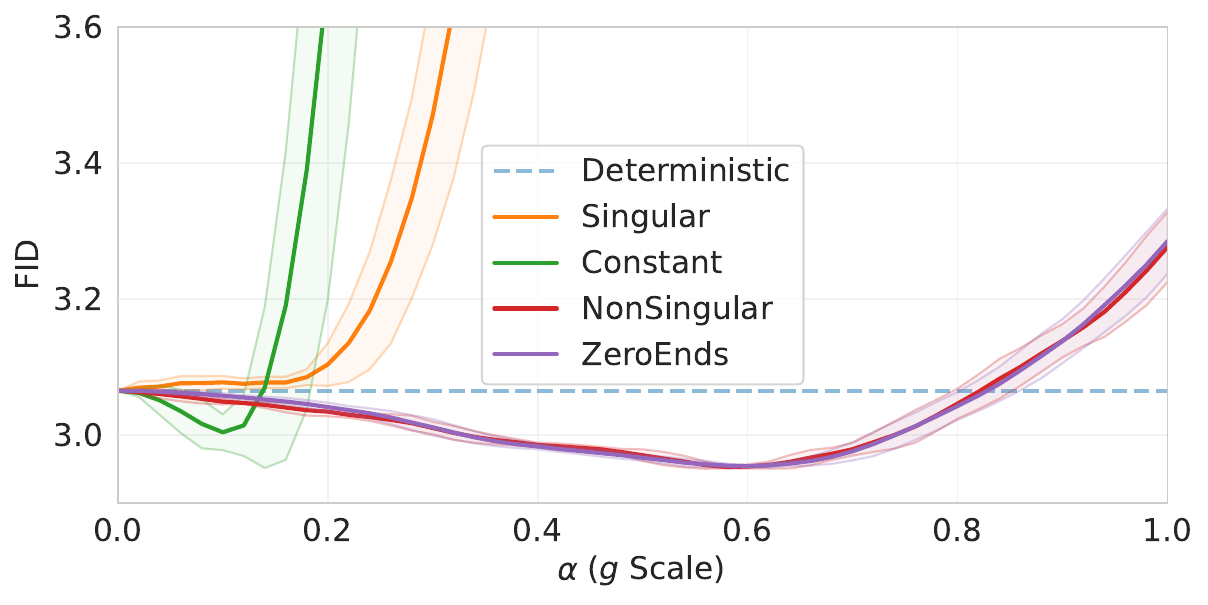}
\caption{64 $\times$ 64}
\label{fig:alpha_comparison_crop_64}
\end{subfigure}
\begin{subfigure}[b]{0.48\linewidth}
\includegraphics[width=\textwidth]{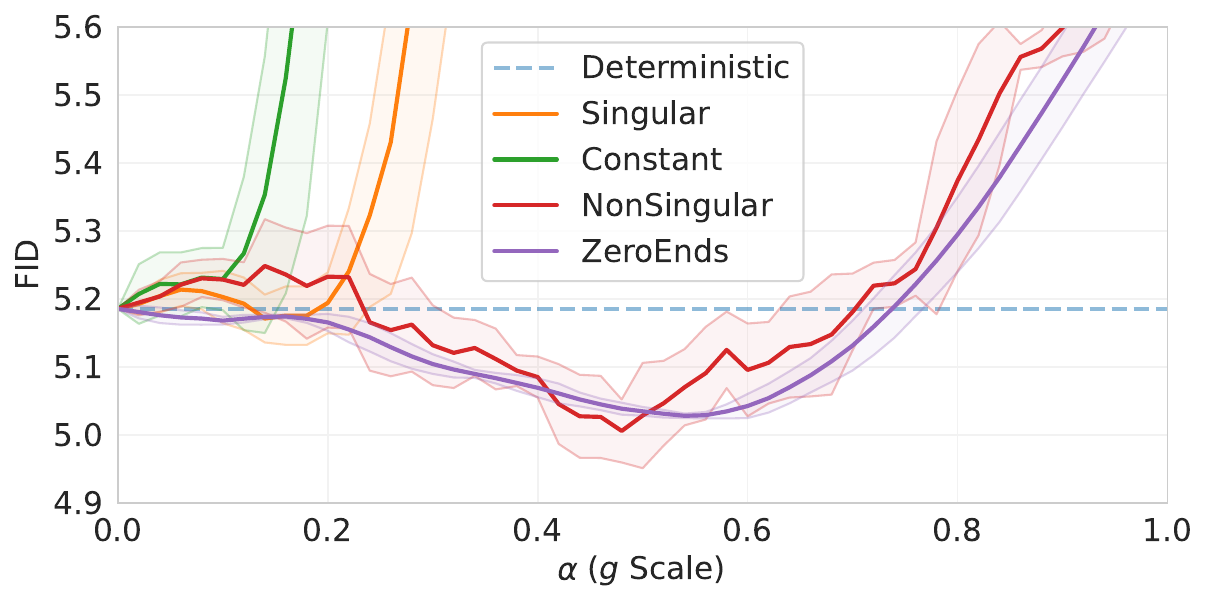}
\caption{128 $\times$ 128}
\label{fig:alpha_comparison_crop_128}
\end{subfigure}
\caption{%
    \textbf{Non-singular samplers work well over a broad range of $\alpha$.}
    Plots of FID for each sampler as the diffusion coefficient scale $\alpha$ is increased.
    Note that at $\alpha=0$ all samplers coincide.
    See \cref{fig:imagenet_fid_alpha_uncropped} for a larger range of FIDs.
}
\label{fig:imagenet_fid_alpha}
\end{figure}

\paragraph{Stochasticity makes FID robust to discretization.}
In \cref{fig:imagenet_fid_numsteps} we compare the effect of the number of sampling steps on FID for various samplers at two image resolutions.
We set $\alpha^2$ proportional to the number of sampling steps with the maximum value provided by \cref{tab:best_fid}.
Again the non-singular samplers perform better than Deterministic at all discretization levels.

\paragraph{Stochastic sampling improves diversity at all classifier-free guidance levels.}
In \cref{fig:imagenet_cfg,fig:qual_alpha_vs_lambda} we show samples from NonSingular using classifier-free guidance (\cref{sec:score_function}), varying both $\alpha$ and $\lambda$, the guidance weight.
In all cases, we can see that diversity increases with $\alpha$, and class typicality increases with $\lambda$.

\paragraph{Qualitative comparisons.}
For qualitative comparisons, we visualize a few samples at various diffusion coefficient scales using different SDEs in \cref{fig:constantg_scale,fig:nonsingular_scale,fig:singular_scale}.
All samples in a column are generated by starting at the same draw $x_1 \sim p_1(x_1)$; different columns start from different draws.
Noise scale $\alpha$ gets progressively larger as we move down the rows.
For Constant, we observe that samples get increasingly noisy with increasing $\alpha$ indicating accumulating errors with increasing scale.
The samples from NonSingular look better, as expected from \cref{fig:imagenet_fid_alpha}.
Lastly, samples from Singular change much more rapidly in comparison to the other samplers, indicating that the singularities in the SDE coefficients increase the effect of noise.

\begin{figure}[tbp]
\centering
\begin{subfigure}[b]{0.48\linewidth}
\includegraphics[width=\linewidth]{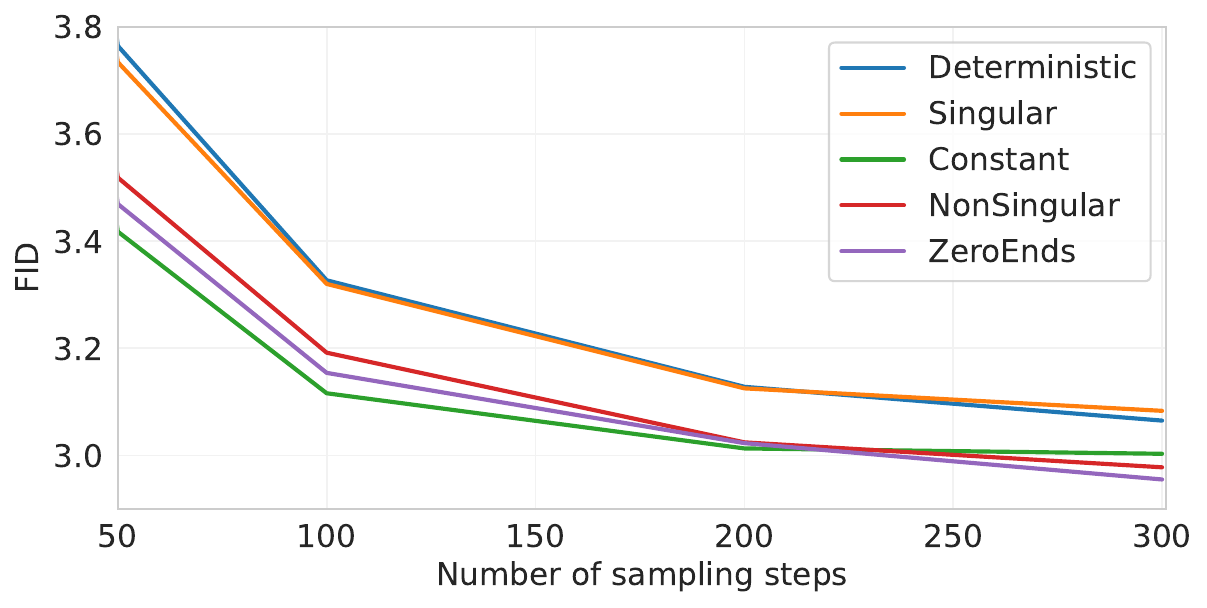}
\caption{64 $\times$ 64}
\label{fig:imagenet_fid_numsteps_64}
\end{subfigure}
\begin{subfigure}[b]{0.48\linewidth}
\includegraphics[width=\textwidth]{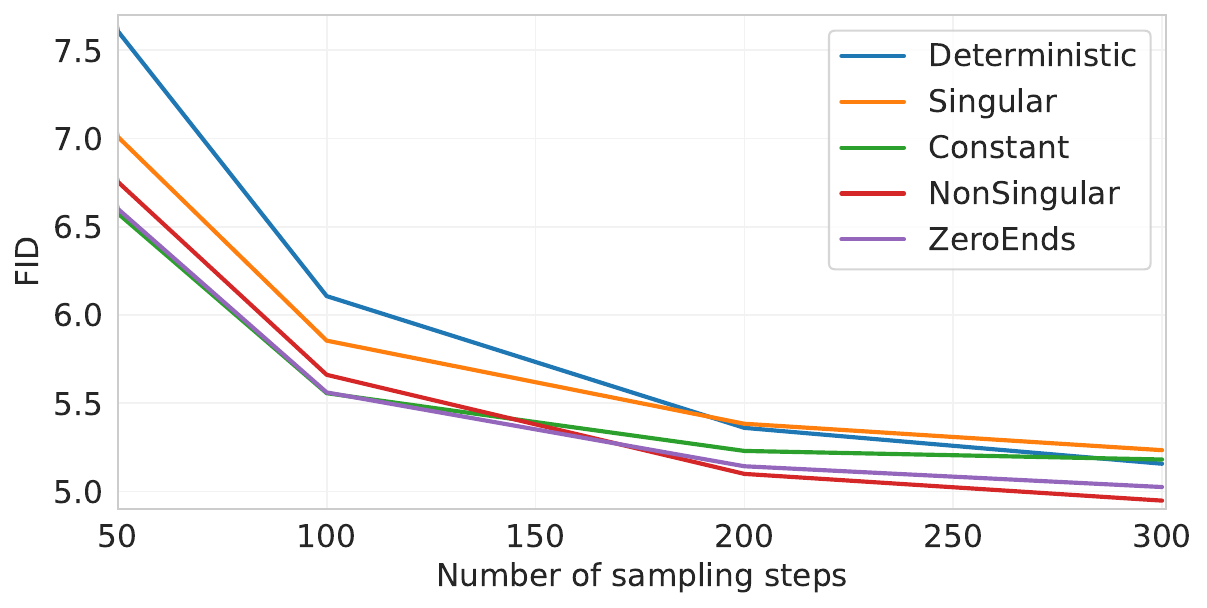}
\caption{128 $\times$ 128}
\label{fig:imagenet_fid_numsteps_128}
\end{subfigure}
\caption{%
    \textbf{Stochasticity makes FID robust to discretization.}
    We compare the effect of number of sampling steps on FID.
    Deterministic is always worse than the non-singular samplers.
}
\label{fig:imagenet_fid_numsteps}
\vspace{-0.4cm}
\end{figure}

\section{Related work}

Transport learning methods learn a mapping between two distributions, where the learned model can transform a sample from one distribution into a sample from the other one.
Typically, one of the distributions is easy to sample (such as a Gaussian) and the other one is the data distribution that one is interested in modeling.
The learned mapping can either be deterministic or stochastic.
A thorough overview of related areas can be found in \citet{yang2024diffusion}.

\paragraph{Deterministic transport.}
Deterministic transport methods implement a change of variable, either explicitly or approximately, that can be used to uniquely map a sample from one distribution to the other.
The normalizing flow family~\citep{rezende2015variational,dinh2017density,kingma2018glow} of methods construct an explicit invertible model that realizes this map either in one step or a finite number of discrete steps.
Neural ODEs~\citep{chen2018neural,grathwohl2018scalable} generalize from discrete steps to a continuous time mapping by inferring and learning the gradient field for all times.
However, Neural ODEs are difficult to train due to the need for simulating the ODEs as part of the training.
Rectified flows, flow matching, and iterative denoising methods~\citep{liu2022flow,lipman2022flow,tong2023improving,heitz2023iterative,delbracio2023inversion} either implicitly or explicitly specify a continuous mapping and learn a model to approximate the continuous time mapping.
Similarly, probability flow ODEs~\citep{song2020score} learned by diffusion models~\citep{sohl2015deep} approximate an implicitly defined continuous mapping.
Our work is useful for flexible sampling from such pre-trained continuous time deterministic Gaussian flows, or more generally where the score function for all the marginal distributions is either provided or can be deduced from the learned flow model.

\paragraph{Stochastic transport.}
Stochastic transport methods learn a stochastic mapping, where a sample from one distribution gets stochastically mapped to a sample from the other.
Gaussian diffusion models are a salient example of such discrete~\cite{sohl2015deep,ho2020denoising} or continuous time~\cite{song2020score,kingma2021variational} mappings where one of the distributions is constrained to be Gaussian.
Several generalizations have been proposed that extend from Gaussian to more general families of distributions~\cite{yoon2024score}.
The stochastic interpolants framework~\citep{albergo2022building,albergo2023stochastic,ma2024sit} further generalizes to a larger family of distributions by introducing a random latent variable allowing efficient estimation of the score function at all times.
Our work is directly applicable to models learned with such methods where the score function is accessible and can be used to construct and explore a large family of samplers.
The convergence rates of diffusion models have been studied in \cite{chen2023score,benton2023linear} with respect to the number of data samples and dimensionality.
However, since our method does not require retraining, it does not affect these properties of the original training algorithms.

\paragraph{Schr\"odinger bridge and optimal transport.}
These methods consider a more general problem of learning transport maps with additional constraints.
k-Rectified flows~\cite{liu2022flow} provide an iterative procedure for tackling deterministic optimal transports for a family of costs, while the more general Schr\"odinger bridge problem, viewed as an entropy regularized optimal transport, is an active area of research \cite{shi2024diffusion,liu2023i2sb}.
Our work is complementary to these methods as we focus on flexible sampling, given the access to the score function for the marginal distributions.

\section{Conclusion}
\label{sec:conclusion}

We introduced a general method to identify a family of SDEs that have the same marginal distribution as a particular SDE, including the special case where the diffusion coefficient of the given SDE is zero.
This special case corresponds to flow models which naively only support deterministic sampling.
Our method enables flexible construction of stochastic samplers for such deterministic models where the diffusion coefficient can be chosen at sampling time from an infinitely large set of possibilities in an application and evaluation metric dependent way.
Our method requires explicit access to the score function, in absence of which it is limited to a subset of flow models where the score function can be derived from the given flow model.
However, this set includes currently popular rectified flow and diffusion models where one of the distributions is Gaussian.

\bibliography{arxiv_main.bib}

\begin{thebibliography}{48}
\providecommand{\natexlab}[1]{#1}
\providecommand{\url}[1]{\texttt{#1}}
\expandafter\ifx\csname urlstyle\endcsname\relax
  \providecommand{\doi}[1]{doi: #1}\else
  \providecommand{\doi}{doi: \begingroup \urlstyle{rm}\Url}\fi

\bibitem[Albergo \& Vanden-Eijnden(2022)Albergo and
  Vanden-Eijnden]{albergo2022building}
Michael~S Albergo and Eric Vanden-Eijnden.
\newblock Building normalizing flows with stochastic interpolants.
\newblock \emph{arXiv preprint arXiv:2209.15571}, 2022.

\bibitem[Albergo et~al.(2023)Albergo, Boffi, and
  Vanden-Eijnden]{albergo2023stochastic}
Michael~S Albergo, Nicholas~M Boffi, and Eric Vanden-Eijnden.
\newblock Stochastic interpolants: A unifying framework for flows and
  diffusions.
\newblock \emph{arXiv preprint arXiv:2303.08797}, 2023.

\bibitem[Anderson(1982)]{anderson1982reverse}
Brian~DO Anderson.
\newblock Reverse-time diffusion equation models.
\newblock \emph{Stochastic Processes and their Applications}, 12\penalty0
  (3):\penalty0 313--326, 1982.

\bibitem[Benton et~al.(2023)Benton, De~Bortoli, Doucet, and
  Deligiannidis]{benton2023linear}
Joe Benton, Valentin De~Bortoli, Arnaud Doucet, and George Deligiannidis.
\newblock Linear convergence bounds for diffusion models via stochastic
  localization.
\newblock \emph{arXiv preprint arXiv:2308.03686}, 2023.

\bibitem[Berner et~al.(2022)Berner, Richter, and Ullrich]{berner2022optimal}
Julius Berner, Lorenz Richter, and Karen Ullrich.
\newblock An optimal control perspective on diffusion-based generative
  modeling.
\newblock \emph{arXiv preprint arXiv:2211.01364}, 2022.

\bibitem[Bradbury et~al.(2018)Bradbury, Frostig, Hawkins, Johnson, Leary,
  Maclaurin, Necula, Paszke, Vander{P}las, Wanderman-{M}ilne, and
  Zhang]{jax2018github}
James Bradbury, Roy Frostig, Peter Hawkins, Matthew~James Johnson, Chris Leary,
  Dougal Maclaurin, George Necula, Adam Paszke, Jake Vander{P}las, Skye
  Wanderman-{M}ilne, and Qiao Zhang.
\newblock {JAX}: composable transformations of {P}ython+{N}um{P}y programs,
  2018.
\newblock URL \url{http://github.com/google/jax}.

\bibitem[Chen et~al.(2023)Chen, Huang, Zhao, and Wang]{chen2023score}
Minshuo Chen, Kaixuan Huang, Tuo Zhao, and Mengdi Wang.
\newblock Score approximation, estimation and distribution recovery of
  diffusion models on low-dimensional data.
\newblock In \emph{International Conference on Machine Learning}, pp.\
  4672--4712. PMLR, 2023.

\bibitem[Chen et~al.(2018)Chen, Rubanova, Bettencourt, and
  Duvenaud]{chen2018neural}
Ricky T.~Q. Chen, Yulia Rubanova, Jesse Bettencourt, and David~K Duvenaud.
\newblock Neural ordinary differential equations.
\newblock In S.~Bengio, H.~Wallach, H.~Larochelle, K.~Grauman, N.~Cesa-Bianchi,
  and R.~Garnett (eds.), \emph{Advances in Neural Information Processing
  Systems}, volume~31. Curran Associates, Inc., 2018.
\newblock URL
  \url{https://proceedings.neurips.cc/paper_files/paper/2018/file/69386f6bb1dfed68692a24c8686939b9-Paper.pdf}.

\bibitem[Dao et~al.(2023)Dao, Phung, Nguyen, and Tran]{dao2023flow}
Quan Dao, Hao Phung, Binh Nguyen, and Anh Tran.
\newblock Flow matching in latent space.
\newblock \emph{arXiv preprint arXiv:2307.08698}, 2023.

\bibitem[Delbracio \& Milanfar(2023)Delbracio and
  Milanfar]{delbracio2023inversion}
Mauricio Delbracio and Peyman Milanfar.
\newblock Inversion by direct iteration: An alternative to denoising diffusion
  for image restoration.
\newblock \emph{arXiv preprint arXiv:2303.11435}, 2023.

\bibitem[Deng et~al.(2009)Deng, Dong, Socher, Li, Li, and
  Fei-Fei]{deng2009imagenet}
Jia Deng, Wei Dong, Richard Socher, Li-Jia Li, Kai Li, and Li~Fei-Fei.
\newblock Imagenet: A large-scale hierarchical image database.
\newblock In \emph{2009 IEEE conference on computer vision and pattern
  recognition}, pp.\  248--255. Ieee, 2009.

\bibitem[Dinh et~al.(2017)Dinh, Sohl-Dickstein, and Bengio]{dinh2017density}
Laurent Dinh, Jascha Sohl-Dickstein, and Samy Bengio.
\newblock Density estimation using real {NVP}.
\newblock In \emph{International Conference on Learning Representations}, 2017.
\newblock URL \url{https://openreview.net/forum?id=HkpbnH9lx}.

\bibitem[Esser et~al.(2024)Esser, Kulal, Blattmann, Entezari, M{\"u}ller,
  Saini, Levi, Lorenz, Sauer, Boesel, et~al.]{esser2024scaling}
Patrick Esser, Sumith Kulal, Andreas Blattmann, Rahim Entezari, Jonas
  M{\"u}ller, Harry Saini, Yam Levi, Dominik Lorenz, Axel Sauer, Frederic
  Boesel, et~al.
\newblock Scaling rectified flow transformers for high-resolution image
  synthesis.
\newblock \emph{arXiv preprint arXiv:2403.03206}, 2024.

\bibitem[Grathwohl et~al.(2019)Grathwohl, Chen, Bettencourt, and
  Duvenaud]{grathwohl2018scalable}
Will Grathwohl, Ricky T.~Q. Chen, Jesse Bettencourt, and David Duvenaud.
\newblock Scalable reversible generative models with free-form continuous
  dynamics.
\newblock In \emph{International Conference on Learning Representations}, 2019.
\newblock URL \url{https://openreview.net/forum?id=rJxgknCcK7}.

\bibitem[Heitz et~al.(2023)Heitz, Belcour, and Chambon]{heitz2023iterative}
Eric Heitz, Laurent Belcour, and Thomas Chambon.
\newblock Iterative $\alpha$-(de) blending: A minimalist deterministic
  diffusion model.
\newblock In \emph{ACM SIGGRAPH 2023 Conference Proceedings}, pp.\  1--8, 2023.

\bibitem[Heusel et~al.(2017)Heusel, Ramsauer, Unterthiner, Nessler, and
  Hochreiter]{heusel2017gans}
Martin Heusel, Hubert Ramsauer, Thomas Unterthiner, Bernhard Nessler, and Sepp
  Hochreiter.
\newblock Gans trained by a two time-scale update rule converge to a local nash
  equilibrium.
\newblock \emph{Advances in neural information processing systems}, 30, 2017.

\bibitem[Ho \& Salimans(2022)Ho and Salimans]{ho2022classifier}
Jonathan Ho and Tim Salimans.
\newblock Classifier-free diffusion guidance.
\newblock \emph{arXiv preprint arXiv:2207.12598}, 2022.

\bibitem[Ho et~al.(2020)Ho, Jain, and Abbeel]{ho2020denoising}
Jonathan Ho, Ajay Jain, and Pieter Abbeel.
\newblock Denoising diffusion probabilistic models.
\newblock \emph{Advances in neural information processing systems},
  33:\penalty0 6840--6851, 2020.

\bibitem[Hoogeboom et~al.(2023)Hoogeboom, Heek, and
  Salimans]{hoogeboom2023simple}
Emiel Hoogeboom, Jonathan Heek, and Tim Salimans.
\newblock simple diffusion: End-to-end diffusion for high resolution images.
\newblock In \emph{International Conference on Machine Learning}, pp.\
  13213--13232. PMLR, 2023.

\bibitem[Huang et~al.(2021)Huang, Lim, and Courville]{huang2021variational}
Chin-Wei Huang, Jae~Hyun Lim, and Aaron~C Courville.
\newblock A variational perspective on diffusion-based generative models and
  score matching.
\newblock \emph{Advances in Neural Information Processing Systems},
  34:\penalty0 22863--22876, 2021.

\bibitem[Isola et~al.(2017)Isola, Zhu, Zhou, and Efros]{isola2017image}
Phillip Isola, Jun-Yan Zhu, Tinghui Zhou, and Alexei~A. Efros.
\newblock Image-to-image translation with conditional adversarial networks.
\newblock In \emph{Proceedings of the IEEE Conference on Computer Vision and
  Pattern Recognition (CVPR)}, July 2017.

\bibitem[Kawar et~al.(2022)Kawar, Elad, Ermon, and Song]{kawar2022denoising}
Bahjat Kawar, Michael Elad, Stefano Ermon, and Jiaming Song.
\newblock Denoising diffusion restoration models, 2022.

\bibitem[Kingma et~al.(2021)Kingma, Salimans, Poole, and
  Ho]{kingma2021variational}
Diederik Kingma, Tim Salimans, Ben Poole, and Jonathan Ho.
\newblock Variational diffusion models.
\newblock \emph{Advances in neural information processing systems},
  34:\penalty0 21696--21707, 2021.

\bibitem[Kingma \& Ba(2014)Kingma and Ba]{kingma2014adam}
Diederik~P Kingma and Jimmy Ba.
\newblock Adam: A method for stochastic optimization.
\newblock \emph{arXiv preprint arXiv:1412.6980}, 2014.

\bibitem[Kingma \& Dhariwal(2018)Kingma and Dhariwal]{kingma2018glow}
Durk~P Kingma and Prafulla Dhariwal.
\newblock Glow: Generative flow with invertible 1x1 convolutions.
\newblock In \emph{Advances in Neural Information Processing Systems},
  volume~31. Curran Associates, Inc., 2018.

\bibitem[Lipman et~al.(2022)Lipman, Chen, Ben-Hamu, Nickel, and
  Le]{lipman2022flow}
Yaron Lipman, Ricky~TQ Chen, Heli Ben-Hamu, Maximilian Nickel, and Matt Le.
\newblock Flow matching for generative modeling.
\newblock \emph{arXiv preprint arXiv:2210.02747}, 2022.

\bibitem[Liu et~al.(2023)Liu, Vahdat, Huang, Theodorou, Nie, and
  Anandkumar]{liu2023i2sb}
Guan-Horng Liu, Arash Vahdat, De-An Huang, Evangelos~A Theodorou, Weili Nie,
  and Anima Anandkumar.
\newblock I{$^2$}sb: Image-to-image schr{\"o}dinger bridge.
\newblock \emph{arXiv preprint arXiv:2302.05872}, 2023.

\bibitem[Liu et~al.(2022)Liu, Gong, and Liu]{liu2022flow}
Xingchao Liu, Chengyue Gong, and Qiang Liu.
\newblock Flow straight and fast: Learning to generate and transfer data with
  rectified flow.
\newblock \emph{arXiv preprint arXiv:2209.03003}, 2022.

\bibitem[Loshchilov \& Hutter(2017)Loshchilov and
  Hutter]{loshchilov2017decoupled}
Ilya Loshchilov and Frank Hutter.
\newblock Decoupled weight decay regularization.
\newblock \emph{arXiv preprint arXiv:1711.05101}, 2017.

\bibitem[Lu et~al.(2022)Lu, Zhou, Bao, Chen, Li, and Zhu]{lu2022dpmsolver}
Cheng Lu, Yuhao Zhou, Fan Bao, Jianfei Chen, Chongxuan Li, and Jun Zhu.
\newblock Dpm-solver: A fast ode solver for diffusion probabilistic model
  sampling in around 10 steps, 2022.

\bibitem[Lugmayr et~al.(2022)Lugmayr, Danelljan, Romero, Yu, Timofte, and
  Van~Gool]{lugmayr2022repaint}
Andreas Lugmayr, Martin Danelljan, Andres Romero, Fisher Yu, Radu Timofte, and
  Luc Van~Gool.
\newblock Repaint: Inpainting using denoising diffusion probabilistic models.
\newblock In \emph{Proceedings of the IEEE/CVF Conference on Computer Vision
  and Pattern Recognition (CVPR)}, pp.\  11461--11471, June 2022.

\bibitem[Ma et~al.(2024)Ma, Goldstein, Albergo, Boffi, Vanden-Eijnden, and
  Xie]{ma2024sit}
Nanye Ma, Mark Goldstein, Michael~S Albergo, Nicholas~M Boffi, Eric
  Vanden-Eijnden, and Saining Xie.
\newblock Sit: Exploring flow and diffusion-based generative models with
  scalable interpolant transformers.
\newblock \emph{arXiv preprint arXiv:2401.08740}, 2024.

\bibitem[Meng et~al.(2022)Meng, He, Song, Song, Wu, Zhu, and
  Ermon]{meng2022sdedit}
Chenlin Meng, Yutong He, Yang Song, Jiaming Song, Jiajun Wu, Jun-Yan Zhu, and
  Stefano Ermon.
\newblock {SDE}dit: Guided image synthesis and editing with stochastic
  differential equations.
\newblock In \emph{International Conference on Learning Representations}, 2022.
\newblock URL \url{https://openreview.net/forum?id=aBsCjcPu_tE}.

\bibitem[Ramesh et~al.(2022)Ramesh, Dhariwal, Nichol, Chu, and
  Chen]{ramesh2022hierarchical}
Aditya Ramesh, Prafulla Dhariwal, Alex Nichol, Casey Chu, and Mark Chen.
\newblock Hierarchical text-conditional image generation with clip latents,
  2022.

\bibitem[Rezende \& Mohamed(2015)Rezende and Mohamed]{rezende2015variational}
Danilo Rezende and Shakir Mohamed.
\newblock Variational inference with normalizing flows.
\newblock In Francis Bach and David Blei (eds.), \emph{Proceedings of the 32nd
  International Conference on Machine Learning}, volume~37 of \emph{Proceedings
  of Machine Learning Research}, pp.\  1530--1538, Lille, France, 07--09 Jul
  2015. PMLR.

\bibitem[Rombach et~al.(2022)Rombach, Blattmann, Lorenz, Esser, and
  Ommer]{rombach2022high}
Robin Rombach, Andreas Blattmann, Dominik Lorenz, Patrick Esser, and Bj\"orn
  Ommer.
\newblock High-resolution image synthesis with latent diffusion models.
\newblock In \emph{Proceedings of the IEEE/CVF Conference on Computer Vision
  and Pattern Recognition (CVPR)}, pp.\  10684--10695, June 2022.

\bibitem[Russakovsky et~al.(2015)Russakovsky, Deng, Su, Krause, Satheesh, Ma,
  Huang, Karpathy, Khosla, Bernstein, Berg, and Fei-Fei]{ILSVRC15}
Olga Russakovsky, Jia Deng, Hao Su, Jonathan Krause, Sanjeev Satheesh, Sean Ma,
  Zhiheng Huang, Andrej Karpathy, Aditya Khosla, Michael Bernstein,
  Alexander~C. Berg, and Li~Fei-Fei.
\newblock {ImageNet Large Scale Visual Recognition Challenge}.
\newblock \emph{International Journal of Computer Vision (IJCV)}, 115\penalty0
  (3):\penalty0 211--252, 2015.
\newblock \doi{10.1007/s11263-015-0816-y}.

\bibitem[Saharia et~al.(2022)Saharia, Chan, Saxena, Li, Whang, Denton,
  Ghasemipour, Ayan, Mahdavi, Lopes, Salimans, Ho, Fleet, and
  Norouzi]{saharia2022photorealistic}
Chitwan Saharia, William Chan, Saurabh Saxena, Lala Li, Jay Whang, Emily
  Denton, Seyed Kamyar~Seyed Ghasemipour, Burcu~Karagol Ayan, S.~Sara Mahdavi,
  Rapha~Gontijo Lopes, Tim Salimans, Jonathan Ho, David~J Fleet, and Mohammad
  Norouzi.
\newblock Photorealistic text-to-image diffusion models with deep language
  understanding, 2022.

\bibitem[S{\"a}rkk{\"a} \& Solin(2019)S{\"a}rkk{\"a} and
  Solin]{sarkka2019applied}
Simo S{\"a}rkk{\"a} and Arno Solin.
\newblock \emph{Applied stochastic differential equations}, volume~10.
\newblock Cambridge University Press, 2019.

\bibitem[Shi et~al.(2024)Shi, De~Bortoli, Campbell, and
  Doucet]{shi2024diffusion}
Yuyang Shi, Valentin De~Bortoli, Andrew Campbell, and Arnaud Doucet.
\newblock Diffusion schr{\"o}dinger bridge matching.
\newblock \emph{Advances in Neural Information Processing Systems}, 36, 2024.

\bibitem[Sohl-Dickstein et~al.(2015)Sohl-Dickstein, Weiss, Maheswaranathan, and
  Ganguli]{sohl2015deep}
Jascha Sohl-Dickstein, Eric Weiss, Niru Maheswaranathan, and Surya Ganguli.
\newblock Deep unsupervised learning using nonequilibrium thermodynamics.
\newblock In \emph{International conference on machine learning}, pp.\
  2256--2265. PMLR, 2015.

\bibitem[Song et~al.(2020)Song, Sohl-Dickstein, Kingma, Kumar, Ermon, and
  Poole]{song2020score}
Yang Song, Jascha Sohl-Dickstein, Diederik~P Kingma, Abhishek Kumar, Stefano
  Ermon, and Ben Poole.
\newblock Score-based generative modeling through stochastic differential
  equations.
\newblock \emph{arXiv preprint arXiv:2011.13456}, 2020.

\bibitem[Tong et~al.(2023)Tong, Malkin, Huguet, Zhang, Rector-Brooks, Fatras,
  Wolf, and Bengio]{tong2023improving}
Alexander Tong, Nikolay Malkin, Guillaume Huguet, Yanlei Zhang, Jarrid
  Rector-Brooks, Kilian Fatras, Guy Wolf, and Yoshua Bengio.
\newblock Improving and generalizing flow-based generative models with
  minibatch optimal transport.
\newblock \emph{arXiv preprint arXiv:2302.00482}, 2023.

\bibitem[Vincent(2011)]{vincent2011connection}
Pascal Vincent.
\newblock A connection between score matching and denoising autoencoders.
\newblock \emph{Neural computation}, 23\penalty0 (7):\penalty0 1661--1674,
  2011.

\bibitem[Xie et~al.(2024)Xie, Zhu, Yu, Yang, Cheng, Zhang, Zhang, and
  Zhang]{xie2024reflected}
Tianyu Xie, Yu~Zhu, Longlin Yu, Tong Yang, Ziheng Cheng, Shiyue Zhang, Xiangyu
  Zhang, and Cheng Zhang.
\newblock Reflected flow matching.
\newblock \emph{arXiv preprint arXiv:2405.16577}, 2024.

\bibitem[Yang et~al.(2024)Yang, Zhang, Song, Hong, Xu, Zhao, Zhang, Cui, and
  Yang]{yang2024diffusion}
Ling Yang, Zhilong Zhang, Yang Song, Shenda Hong, Runsheng Xu, Yue Zhao, Wentao
  Zhang, Bin Cui, and Ming-Hsuan Yang.
\newblock Diffusion models: A comprehensive survey of methods and applications,
  2024.

\bibitem[Yoon et~al.(2024)Yoon, Park, Kim, and Lim]{yoon2024score}
Eun~Bi Yoon, Keehun Park, Sungwoong Kim, and Sungbin Lim.
\newblock Score-based generative models with l{\'e}vy processes.
\newblock \emph{Advances in Neural Information Processing Systems}, 36, 2024.

\bibitem[Zheng et~al.(2023)Zheng, Le, Shaul, Lipman, Grover, and
  Chen]{zheng2023guided}
Qinqing Zheng, Matt Le, Neta Shaul, Yaron Lipman, Aditya Grover, and Ricky~TQ
  Chen.
\newblock Guided flows for generative modeling and decision making.
\newblock \emph{arXiv preprint arXiv:2311.13443}, 2023.

\end{thebibliography}
\bibliographystyle{iclr2025_conference}


\clearpage
\appendix
\section{Derivation of singular SDE}
\label{sec:singular_sde_app}

We consider the following SDE with additive noise; i.e., the diffusion coefficient $g$ is only a function of time.
\begin{align}
dx = f(x, t)dt + g(t)dW_t
\end{align}
The perturbation kernel $p(x_t | x_0)$ corresponding to rectified flow is Gaussian, with $p(x_t | x_0) = N(x_t; (1-t)x_0+t\mu_1, t^2\sigma_1^2I)$.
Since the perturbation kernel is Gaussian, following \citet{song2020score}, we assume that the drift term is affine; i.e. $f(x, t) \equiv f(t)x$. 
Further since $X_0, X_1$ are independent, we can directly infer the first and second moments $\mu_t, \Sigma_t$ for the marginals $p_t(x)$ as $\mu_t =  (1-t)\mu_0 + t\mu_1$ and $\Sigma_t = (1-t)^2\Sigma_0 + t^2\sigma_1^2I$.

From Eq (5.50) of \citet{sarkka2019applied} we have
\begin{align}
\frac{d\mu_t}{dt} &= \E_{p_t(x)} [f(t)x] \\
&= f(t)\mu_t
\end{align}
where $\mu_t$ is the mean at time $t$. Rearranging and integrating both sides:
\begin{align}
\ln \frac{\mu_t}{\mu_0} &= \int_0^t f(s)ds & \\
\ln \frac{(1-t)\mu_0+t\mu_1}{\mu_0} &= \int_0^t f(s)ds & \text{Substituting}~\mu_t = (1-t)\mu_0+t\mu_1\\
\frac{\mu_1 - \mu_0}{(1-t)\mu_0 + t\mu_1} &= f(t)  &\text{Differentiating both sides w.r.t. $t$}\\
\end{align}
Substituting $\mu_1=0$, we get as in \cref{eq:singular_choices}:
\begin{align}
\label{eq:singular_f}
f(x, t) &= -\frac{x}{1-t}
\end{align}
Similarly, from Eq. (5.51) of \citet{sarkka2019applied}:
\begin{align}
\frac{d\Sigma_t}{dt} &= \E_{p_t(x)}\left[f(x,t)(x-\mu_t)^T + (x - \mu_t)f(x, t)^T + G(x, t)QG(x, t)^T\right]
\end{align}
Substituting $Q \equiv I$ (we are assuming isotropic dispersion), $G(x, t) \equiv g(t)I$ (symmetric, time-dependent diffusion coefficient), and $f(x,t)$ from \cref{eq:singular_f}:
\begin{align}
\frac{d\Sigma_t}{dt} &= \E_{p_t(x)}\left[-\frac{x}{1-t}(x-\mu_t)^T - (x - \mu_t)\frac{x^T}{1-t} + g^2(t)I\right] \\
&= \frac{2}{1-t}\E_{p_t(x)}\left[-xx^T + \mu_t\mu_t^T\right] + g^2(t)I \\
&= -\frac{2\Sigma_t}{1-t} + g^2(t)I \\
\label{eq:inhomo_sigma_diff}
\implies \frac{d\Sigma_t}{dt} + \frac{2\Sigma_t}{1-t} &= g^2(t)I
\end{align}
Above is an inhomogenous differential equation.
The integrating factor $I(t)$ can be calculated as:
\begin{align}
I(t) &= \exp \left(\int_0^t \frac{2}{1-s}ds \right) = \frac{1}{(1-t)^2}
\end{align}
Multiplying both sides of \cref{eq:inhomo_sigma_diff}, we can write:
\begin{align}
\frac{d}{dt}\left[ \frac{\Sigma_t}{(1-t)^2} \right] &= \frac{g^2(t)I}{(1-t)^2}
\end{align}
Integrating both sides:
\begin{align}
\left[ \frac{\Sigma_s}{(1-s)^2} \right]_0^t &= \int_0^t\frac{g^2(s)I}{(1-s)^2}ds \\
\frac{\Sigma_t}{(1-t)^2} - \Sigma_0 &= \int_0^t\frac{g^2(s)I}{(1-s)^2}ds
\end{align}
Substituting $\Sigma_t = (1-t)^2\Sigma_0 + t^2\sigma_1^2I$:
\begin{align}
\frac{(1-t)^2\Sigma_0 + t^2\sigma_1^2I}{(1-t)^2} - \Sigma_0^2 &= \int_0^t\frac{g^2(s)I}{(1-s)^2}ds
\end{align}
Differentiating both sides w.r.t. $t$ and simplifying yields:
\begin{align}
g^2(t) = \frac{2t\sigma_1^2}{1-t}
\end{align}
Substituting $\sigma_1 = 1$ to the result in \cref{eq:singular_choices}:
\begin{align}
g(t) = \sqrt{\frac{2t}{1-t}}
\end{align}

\section{Score function from rectified flow}
\label{sec:score_fn_app}
Given a base data distribution $p(x)$ and a conditional noising distribution $p_\sigma(\tilde  x | x)$, Denoising score matching \cite{vincent2011connection} learns the score for the marginal $p_\sigma(\tilde x)$ by optimizing:
\begin{equation}
\nabla_{\tilde x} \ln p_\sigma(\tilde x) = \argmin_{\psi} \E_{p_\sigma(x_0, \tilde x)} \left[ \frac{1}{2} \left\lVert \psi(\tilde x) - \frac{\partial \ln p_{\sigma}(\tilde x | x_0)}{\partial \tilde x} \right\lVert^2\right]
\end{equation}
where $p_\sigma(x_0, \tilde x) \equiv p(x)p_\sigma(\tilde x| x_0)$.
The solution to the above optimization problem can be written as:
\begin{equation}
\nabla_{\tilde x} \ln p_\sigma(\tilde x) = \E_{p_\sigma(x_0|\tilde x)} \frac{\partial \ln p_{\sigma}(\tilde x | x_0)}{\partial \tilde x}
\end{equation}
Mapping the above to rectified flow with $\sigma \equiv t, \tilde x \equiv x_t$ we get:
\begin{equation}
\label{eq:dsm_solution}
\nabla_{x_t} \ln p_t(x_t) = \E_{p_t(x_0|x_t)} \frac{\partial \ln p_t(x_t | x_0)}{\partial x_t}
\end{equation}
Next if:
\begin{align}
p_t(x_t | x_0) &= N(x_t; \mu(x_0, t), \sigma^2(x_0, t)I) \\
\frac{\partial \ln p_t(x_t | x_0)}{\partial x_t} &= \frac{\partial}{\partial x_t}\frac{-||x_t - \mu(x_0, t)||^2}{2\sigma(x_0, t)^2} \\
&= \frac{-(x_t - \mu(x_0, t))}{\sigma(x_0, t)^2}
\end{align}
Now:
\begin{equation}
\nabla_{x_t} \ln p_t(x_t) = \E_{p_t(x_0|x_t)} \frac{-(x_t - \mu(x_0, t))}{\sigma(x_0, t)^2}
\end{equation}
Next, assume the covariance $\sigma(x_0,t)$ doesn't depend on $x_0$ -- i.e., $\sigma(x_0, t) \equiv \sigma(t)$ -- and the mean $\mu(x_0, t)$ is linear in $x_0$.
Then:
\begin{align}
\E_{p_t(x_0|x_t)} \frac{-(x_t - \mu(x_0, t))}{\sigma(x_0, t)^2} &= \E_{p_t(x_0|x_t)} \frac{-(x_t - \mu(x_0, t))}{\sigma(t)^2} \\
\label{eq:gaussian_score}
&= \frac{-(x_t - \mu(\E[x_0|x_t], t))}{\sigma(t)^2}
\end{align}

\subsection{Gaussian Rectified Flow}
Consider the special case where $x_t = (1-t)x_0 + tx_1, x_1 \sim N(x_1; \mu_1, \sigma_1^2I)$.
We have $p_t(x_t|x_0) = N((1-t)x_0 + t\mu_1, t^2\sigma_1^2)$.
Using the result from \cref{eq:gaussian_score} we get:
\begin{align}
\nabla_{x_t} \ln p_t(x_t) &= \frac{-(x_t - \mu(\E[x_0|x_t], t))}{\sigma(t)^2}
\end{align}
From this, we can write:
\begin{align}
x_1 &= \frac{x_t - (1-t)x_0}{t} \\
\E [x_1-x_0 | x_t] &= \E \left[\frac{x_t - (1-t)x_0}{t} - x_0 \bigg\lvert x_t\right] \\
&= \E \left[\frac{x_t - (1-t)x_0 - tx_0}{t} | x_t\right] \\
&= \E \left[\frac{x_t-x_0}{t} | x_t\right] \\
&= \frac{x_t - \E[x_0|x_t]}{t} \\
\E[x_0|x_t] &= x_t - t\E [x_1-x_0 | x_t] \\
\nabla_{x_t} \ln p_t(x_t) &= \frac{\mu(\E[x_0|x_t], t)-x_t}{\sigma(t)^2} \\
&= \frac{(1-t)\E[x_0|x_t]+t\mu_1-x_t}{t^2\sigma_1^2} \\
&= \frac{(1-t)(x_t - t\E [x_1-x_0 | x_t])+t\mu_1-x_t}{t^2\sigma_1^2} \\
&= \frac{-(1-t)t\E [x_1-x_0 | x_t]+t\mu_1-tx_t}{t^2\sigma_1^2} \\
&= \frac{-(1-t)\E [x_1-x_0 | x_t]+\mu_1-x_t}{t\sigma_1^2} \label{eq:gauss_rect_score}
\end{align}

\subsection{General Rectified Flow}
\label{sec:general_rect_flow}

First recall the change of variables formula for a density $p(x)$ with $y=g(x)$ where $g$ is invertible and $g^{-1}$ is differentiable:
\begin{align}
p(y) &= p(g^{-1}(y))\left\lvert \det \left[\frac{\partial g^{-1}(z)}{\partial z}\right]_{z=y} \right\rvert
\end{align}
Now, with $x_0 \sim p_1(x_0)$ and $x_1 \sim p_1(x_1)$ and $x_0,x_1 \in \R^d$, let $x_t = g(x_1; x_0)$ be a function that is invertible in first argument and whose inverse $g^{-1}(x_t;x_0)$ is differentiable w.r.t. the first argument.
Note that for simple rectified flows, $x_t = (1-t)x_0 + tx_1$ satisfies these conditions.

We can now express the conditional density $p(x_t | x_0)$ as:
\begin{align}
p(x_t | x_0) = p_1(g^{-1}(x_t;x_0))\left\lvert \det \left[\nabla_z g^{-1}(z;x_0)\right]_{z=x_t} \right\rvert
\end{align}
The score for the conditional density can then be calculated as
\begin{align}
 \frac{\partial \ln p_t(x_t | x_0)}{\partial x_t} &= \nabla_z \ln p_1(z)|_{z=g^{-1}(x_t; x_0)} + \nabla_{x_t} \ln \left\lvert \det \left[\nabla_z g^{-1}(z;x_0)\right]_{z=x_t} \right\rvert
\end{align}
and the score for the marginal density as:
\begin{align}
 \nabla_{x_t} \ln p_t(x_t) &= \E_{p_t(x_0|x_t)} \frac{\partial \ln p_t(x_t | x_0)}{\partial x_t} \\
 \label{eq:general_rect_flow_score}
 &= \E_{p_t(x_0|x_t)} \left[\nabla_z \ln p_1(z)|_{z=g^{-1}(x_t; x_0)} + \nabla_{x_t} \ln \left\lvert \det \left[\nabla_z g^{-1}(z;x_0)\right]_{z=x_t} \right\rvert \right]
\end{align}
For the specific case of Rectified flows, define $g(x_1;x) = (1-t)x + tx_1$.
Then:
\begin{align}
x_t &= g(x_1; x_0) \\
g^{-1}(x_t; x_0) &= \frac{x_t - (1-t)x_0}{t} & \quad \text{Inverse is w.r.t. first argument} \\
\frac{\partial g^{-1}(x_t; x_0)}{\partial x_t} &= \frac{1}{t}I \\
\det \frac{1}{t}I &= \frac{1}{t^d} & \quad \text{$I$ is $d \times d$} \\
p_t(x_t | x_0) &= \frac{1}{t^d}p_1\left(\frac{x_t - (1-t)x_0}{t}\right)
\end{align}
Substituting into \cref{eq:general_rect_flow_score}:
\begin{align}
\nabla_{x_t} \ln p_t(x_t) &= \E_{p_t(x_0|x_t)} \left[\frac{1}{t}\nabla_z \ln p_1(z)|_{z=g^{-1}(x_t; x_0)} \right]
\end{align}
It can be verified that with the choice of $p_1(x_1) \equiv N(\mu_1, \sigma_1^2 I)$, we recover \cref{eq:gauss_rect_score}.

\section{Proof of \cref{thm:general}}
\label{sec:mainresult_proof_app}

\mainresult*
\begin{proof}
The evolution of the marginal probability density $p_t(x)$ is then described by the Fokker-Planck-Kolmogorov (FPK) equation \citep{sarkka2019applied} as:
\begin{align}
\label{eq:full_fpk_expanded}
\frac{\partial p_t}{\partial t} = -\sum_{i=1}^{d} \frac{\partial}{\partial x_i}[\FB p_t] + \frac{1}{2}\sum_{i=1}^{d}\sum_{j=1}^{d} \frac{\partial^2}{\partial x_ix_j}\left[\sum_{k=1}^{d}\GB_{ik}\GB_{jk}p_t \right]
\end{align}

We write the above more succinctly as:
\begin{align}
\label{eq:full_fpk}
\frac{\partial p_t}{\partial t} = -\nabla \cdot [\FB p_t] + \frac{1}{2}\nabla \cdot\left[\GB\GB^Tp_t \right] \cdot \nabla^T
\end{align}
Where $\nabla \cdot$ is the divergence operator.
Next for an arbitrary $R \equiv R(x, t)$ consider the following identity:
\begin{align}
\left[RR^Tp_t \right] \cdot \nabla^T &= \nabla\cdot \left[RR^Tp_t \right] & \quad RR^T~\text{ is symmetric}\\
&= \left[\nabla\cdot RR^T \right]p_t  + RR^T \cdot \nabla p_t \\
&= \left[\nabla\cdot RR^T \right]p_t  + RR^T \cdot p_t \nabla\ln  p_t \\
\label{eq:r_identity}
&= \left(\nabla\cdot RR^T  + RR^T \cdot \nabla \ln p_t \right)p_t
\end{align}
Expanding out $\FB p_t$ by substituting for $\FB$:
\begin{align}
\FB p_t &= \left[f - \frac{1}{2}\left(\nabla \cdot [(1-\gamma_t)GG^T-\GT\GT^T] + [(1-\gamma_t)GG^T-\GT\GT^T]\cdot\nabla\ln p_t \right)\right] p_t \\
\begin{split}
&= fp_t - \frac{1-\gamma_t}{2}\left(\nabla \cdot GG^T + GG^T\cdot\nabla\ln p_t \right) p_t \\
&+ \frac{1}{2}\left(\nabla \cdot \GT\GT^T + \GT\GT^T\cdot\nabla\ln p_t \right) p_t
\end{split}
\end{align}
Using \cref{eq:r_identity} and rewriting:
\begin{align}
\begin{split}
\FB p_t &= fp_t - \frac{1-\gamma_t}{2}[GG^Tp_t] \cdot \nabla^T + \frac{1}{2}[\GT\GT^Tp_t] \cdot \nabla^T
\end{split}
\end{align}
Next we revisit \cref{eq:full_fpk}, and substitute for $\FB p_t$ and $\GB$ with $\GB\GB^T = \gamma_t GG^T + \GT\GT^T$:
\begin{align}
\begin{split}
\frac{\partial p_t}{\partial t} &= -\nabla \cdot [fp_t - \frac{1-\gamma_t}{2}[GG^Tp_t] \cdot \nabla^T + \frac{1}{2}[\GT\GT^Tp_t] \cdot \nabla^T ] \\
&+ \frac{1}{2}\nabla \cdot\left[(\gamma_t GG^T + \GT\GT^T)p_t \right] \cdot \nabla^T
\end{split} \\
\begin{split}
&= -\nabla \cdot [fp_t] + \frac{1-\gamma_t}{2}\nabla \cdot [GG^Tp_t] \cdot \nabla^T - \frac{1}{2}\nabla \cdot [\GT\GT^Tp_t] \cdot \nabla^T \\
&+ \frac{\gamma_t}{2}\nabla \cdot\left[GG^Tp_t \right] \cdot \nabla^T + \frac{1}{2}\nabla \cdot\left[\GT\GT^Tp_t \right] \cdot \nabla^T
\end{split}
\end{align}
With cancellations, we arrive at:
\begin{align}
\frac{\partial p_t}{\partial t} &= -\nabla \cdot [fp_t] + \frac{1}{2}\nabla \cdot [GG^Tp_t] \cdot \nabla^T
\end{align}
which is the FPK equation describing the time evolution of the marginal probability density $p_t(x)$ of the solutions of the SDE in \cref{eq:general_sde_paper}.

\end{proof}

\section{Proof of \cref{cor:diffusion_corollary}}
\label{sec:corollaryone_proof_app}

\corollaryone*
\begin{proof}
Starting with Theorem 1, let's define $\tilde G \equiv 0$ and $G \equiv g(t)I$, where $g(t)$ is a scalar valued function. These choices lead to following

\begin{align}
\bar G &= [\gamma_t(g(t)I)^2]^{\frac{1}{2}} = \sqrt{\gamma_t}g(t)I \\
\bar f &= f - \frac{1}{2}\left(\nabla \cdot [(1-\gamma_t)g^2(t)I] + [(1-\gamma_t)g^2(t)I]\cdot\nabla\ln p_t \right)\\
&= f - \frac{1}{2}\left([(1-\gamma_t)g^2(t)I]\cdot\nabla\ln p_t \right) \label{eq:corr1p1-p-1} \\
&= f - \frac{(1-\gamma_t)g^2(t)}{2}\nabla\ln p_t
\end{align}

Note that \cref{eq:corr1p1-p-1} follows from $\nabla \cdot [(1-\gamma_t)g^2(t)I] = 0$ since neither $\gamma_t$ nor $g(t)$ are functions of $x$.
\end{proof}

\section{Proof of \cref{cor:detflow_corollary}}
\label{sec:corollarytwo_proof_app}

\corollarytwo*
\begin{proof}
First note that $f \equiv v(x, t)$ by definition from equation (3) in the paper. Now, again starting with Theorem 1, let's define $\tilde G \equiv \tilde g(t)I$ and $G \equiv 0$, where $\tilde g(t)$ is a scalar valued function. These choices lead to following

\begin{align}
\bar G &= [(\tilde g(t)I)^2]^{\frac{1}{2}} = \tilde g(t)I \\
\bar f &= v(x, t) - \frac{1}{2}\left(\nabla \cdot [-\tilde g^2(t)I] + [-\tilde g^2(t)I]\cdot\nabla\ln p_t \right)\\
&= v(x, t) + \frac{1}{2}\left([\tilde g^2(t)I]\cdot\nabla\ln p_t \right) \label{eq:corr1p2-p-1} \\
&= v(x, t) + \frac{\tilde g^2(t)}{2}\nabla\ln p_t
\end{align}

Note that \cref{eq:corr1p2-p-1} follows from $\nabla \cdot [-\tilde g^2(t)I] = 0$ since $\tilde g(t)$ is not a function of $x$.
\end{proof}

\section{Closed form rectified flow expression for the two Gaussian case}
\label{sec:closed_gaussian_app}

Our empirical studies use a two Gaussian toy problem setup.
We state the closed form expression for the rectified flow for this case.
Consider $x_0 \sim N(\mu_0, \sigma^2_0 I), x_1 \sim N(\mu_1, \sigma^2_1 I)$:
\begin{equation}
x_t = \alpha_t x_0 + \beta_t x_1, \quad \alpha_t > 0, \alpha_0=1, \alpha_1=0, \beta_t > 0, \beta_0=0, \beta_1=1
\end{equation}
The marginal density $p_t(x_t)$ is also Gaussian:
\begin{equation}
p_t(x_t)=  N(x_t; \alpha_t \mu_0 + \beta_t \mu_1, \alpha_t^2\sigma_0^2 + \beta_t^2\sigma_1^2)
\end{equation}
We have:
\begin{align}
v(x, t) = \E [x_1 - x_0 | x_t] \equiv \mathbb E_{p(x_0, x_1 | x_t)} [x_1 - x_0]
\end{align}
Using the following:
\begin{align}
p(x_0, x_1 | x_t) &= \frac{p(x_t| x_0, x_1)p(x_0,x_1)}{p(x_t)} = \frac{p(x_t| x_0, x_1)p_0(x_0)p_1(x_1)}{p(x_t)} \\
p(x_t| x_0, x_1) &= \delta(x_t - (1-t)x_0 - tx_1)
\end{align}
and elementary properties of Gaussian and Dirac delta distributions, it can be verified that:
\begin{equation}
\label{eq:closed_form_gaussian}
v(x, t) = \frac{(k_t\mu_1-x_t)\alpha_t\sigma_0^2 + (x_t-k_t\mu_0)\beta_t\sigma_1^2}{\alpha_t^{2}\sigma_0^2 + \beta_t^{2}\sigma_1^2} 
\end{equation}
where $k_t=\alpha_t + \beta_t$.

\section{Toy Gaussian Experiment Details}
\label{sec:toy_details}

In the experiments in \cref{sec:guass_experiments} we study how various SDEs in \cref{tab:sde_family_table} behave on a toy problem where both $p_0 \equiv N(-1, 0.3)$ and $p_1 \equiv N(0, 1.0)$ are Gaussian.
In this case the marginal distributions $p_t$ for Gaussian flow are Gaussian with $p_t = N(\mu_t, \sigma_t^2)$ and the true statistics $\mu_t, \sigma_t^2$ can be easily computed.
In addition, the rectified flow $v(x, t)$ is available in closed form (see \cref{sec:closed_gaussian_app}).
The SDEs are simulated backwards in time from $t=1$ with draws from $p_1$ using \cref{eq:reverse_sde_song}.
The drift $\tilde f$ and diffusion $\tilde g$ terms are calculated using \cref{tab:sde_family_table} by setting $\alpha=1$ and using the closed form $v(x, t)$ from \cref{eq:closed_form_gaussian}.
We simulate 10 trials of 10000 trajectories using Euler-Maruyama discretization with varying number of steps. Estimates for mean $\mu_t$ and variance $\sigma_t^2$ at each timestep for various SDEs are calculated, along with their standard deviation across trials.
Error, calculated as estimate - truth, is then plotted in \cref{fig:toy_samplers,fig:toy_steps_eval,fig:toy_bias_var} for both the mean and the variance estimates along with the KL-Divergence from the true marginal distribution.

\section{ImageNet experiment training/evaluation details}
\label{sec:imagenet_training_details_app}
We train two base Rectified flow models to yield $v(x, t)$ at two resolutions of 64 $\times$ 64 and 128 $\times$ 128, on the entire ImageNet training dataset containing roughly 1.2 million images.
Our model is based on the architecture described in \citet{hoogeboom2023simple}.
The model is structured such that the lower feature map resolution is 16 $\times$ 16.
Therefore, for 64 $\times$ 64 resolution two downsamplings are performed, while for 128 $\times$ 128 three downsamplings are performed.
The model is trained with SGD using adamw~\citep{kingma2014adam,loshchilov2017decoupled} with $\beta_1=0.9, \beta_2=0.99, \epsilon=10^{-12}$ for $1000$ epochs.
We use center crop and left-right flips as the only augmentations.
An exponential moving average, with a decay of $0.9999$, of parameters is used for evaluation.
FIDs are reported over the training dataset with reference statistics computed with center crop but without any augmentation, but with class conditioning.
Samplers were evaluated for all $\alpha \in \{0.0, 0.02, 0.04, \dots, 1.0\}$.

The $64 \times 64$ model trained for 500 epochs in 4 days, 8 hours on $8 \times 8$ Google Cloud TPUs v3.
The $128 \times 128$ model trained for 500 epochs in 4 days, 20 hours on $8 \times 8$ Google Cloud TPUs v3.

\section{Example Implementation}
\label{sec:code}

See \cref{lst:nonsingular_code} for an example implementation of the NonSingular sampler.

\section{Additional Experimental Results}
\label{sec:additional_experiments}

\begin{figure}[tb]
\centering
\begin{subfigure}[b]{0.48\linewidth}
\includegraphics[width=\textwidth]{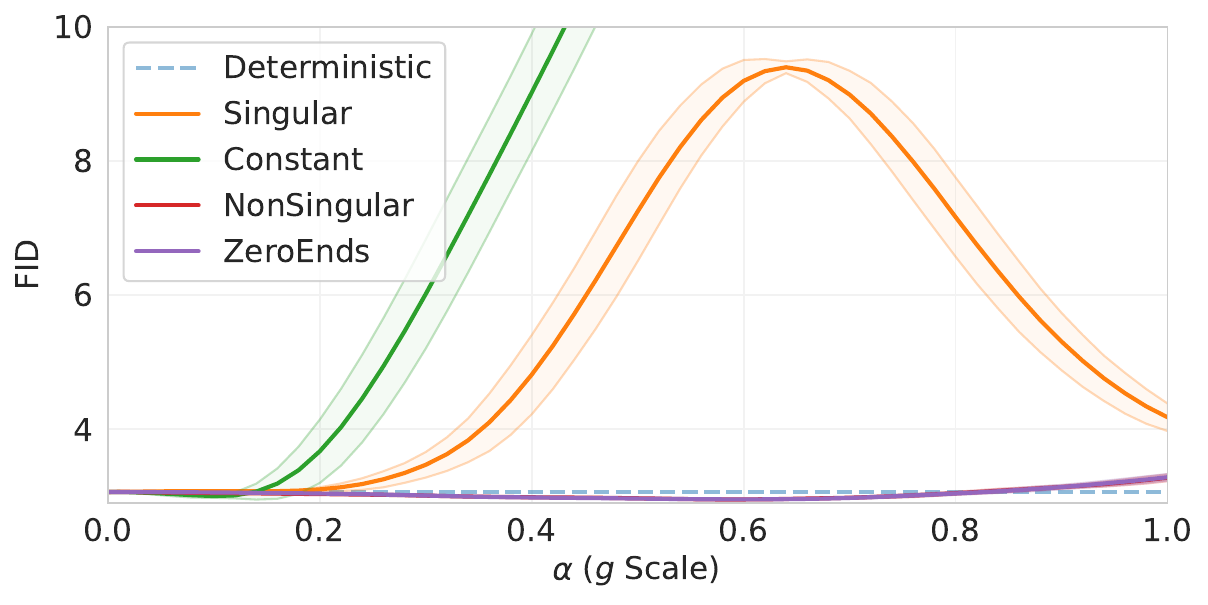}
\caption{64 $\times$ 64}
\label{fig:alpha_comparison_64}
\end{subfigure}
\begin{subfigure}[b]{0.48\linewidth}
\includegraphics[width=\textwidth]{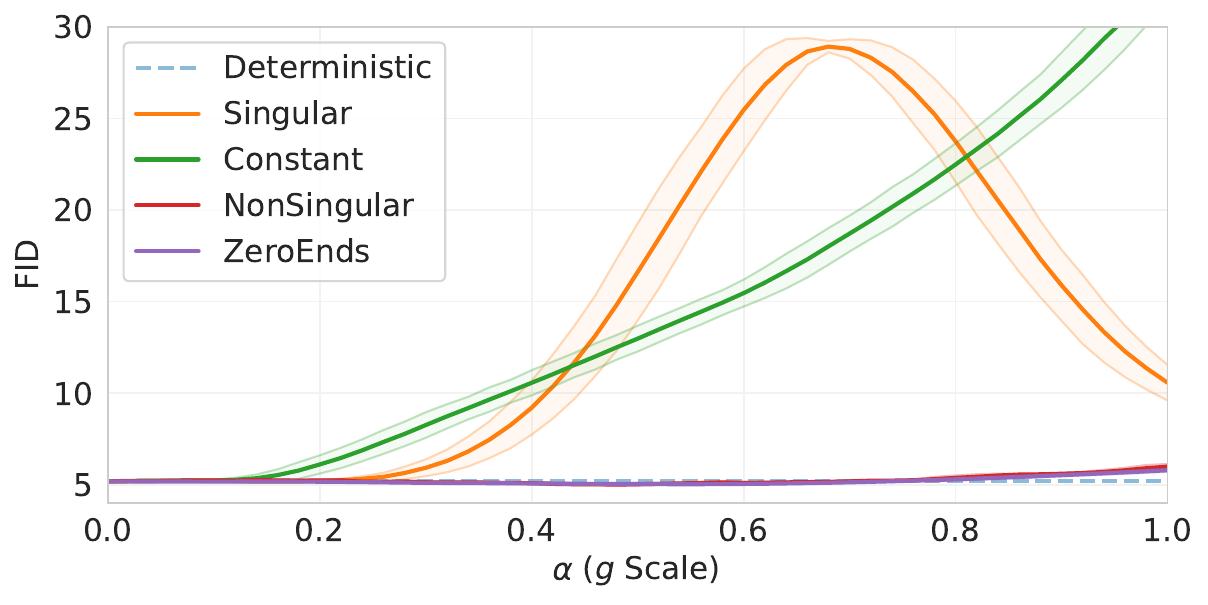}
\caption{128 $\times$ 128}
\label{fig:alpha_comparison_128}
\end{subfigure}
\caption{%
    \textbf{Non-singular samplers work well over a broad range of $\alpha$.}
    The same plots as \cref{fig:imagenet_fid_alpha}, but showing a larger range of FIDs.
    Note that the Singular sampler is highly non-monotonic as a function of $\alpha$.
}
\label{fig:imagenet_fid_alpha_uncropped}
\end{figure}

\subsection{FID vs $\alpha$}

In \cref{fig:imagenet_fid_alpha_uncropped} we show a larger range FID for various samplers compared in \cref{fig:imagenet_fid_alpha}.
We observe that the Singular sampler tends to perform well only at low scales with an intriguing behavior for higher scales where the FID starts to improve again after worsening significantly.

\subsection{Effect of diffusion coefficient magnitude on samples}
\label{sec:diff_coeff_mag_qual}

We qualitatively visualize the effect of diffusion coefficient magnitude for the three SDEs discussed in the main paper.
\cref{fig:constantg_scale} visualizes samples for the constant diffusion term SDE as a function increasing coefficient magnitude.
Each column is a different sample starting with the same random draw from $p_1(x_1)$.
Each row corresponds to a different magnitude for the diffusion coefficient $g(t)$.
\cref{fig:nonsingular_scale,fig:singular_scale} visualize samples with a similar scheme for the non-singular and singular SDE.

\subsection{Classifier-free guidance samples}
\cref{fig:qual_alpha_vs_lambda} shows additional samples using classifier-free guidance with NonSingular at different values of $\alpha$ and $\lambda$, as in \cref{fig:imagenet_cfg}.

\begin{figure}[tbp]
\small
\caption{NonSingular Sampler written in JAX.}
\label{lst:nonsingular_code}
\begin{lstlisting}[language=Python]
def non_singular_sampler(
    rng, num_samples, model, params, labels, g_scale, num_steps=1000,
    batch_size=10, image_size=64, num_channels=3, num_classes=1000,
    n=1, m=0):
  """Draw samples from the model."""
  p_1_samples = []
  p_0_samples = []
  t = jnp.linspace(1., 0., num_steps+1)
  t_ones = jnp.ones([batch_size, 1, 1])

  # Sampler loop body
  def body_fn(i, z):
    z, labels, rng = z
    tb = t[i] * t_ones
    z_hat = model.apply({'params': params}, z, (1 - tb), labels)
    v = -z_hat
    b = g_scale
    g = b * jnp.power(tb, n / 2) * jnp.power(1 - tb, m / 2)
    s_u = -((1-tb) * v + z)
    fr = (v - jnp.square(b) * jnp.power(tb, n-1
          * jnp.power(1-tb, m) * s_u / 2)
    rng, key = jax.random.split(rng)
    dt = t[i+1] - t[i]
    dbt = (jnp.sqrt(jnp.abs(dt))
           * jax.random.normal(key, shape=z.shape))
    z = z + fr * dt + g * dbt
    return z, labels, rng

  max_steps = num_samples // batch_size
  for _ in range(max_steps):
    # Sample from p_1
    rng, key = jax.random.split(rng)
    z = sample_from_prior(
      key, shape=[batch_size, image_size, image_size, num_channels])
    p_1_samples.append(z)

    # Run the sampler
    rng, key = jax.random.split(rng)
    init_val = (z, labels, key)
    z, _, _ = jax.lax.fori_loop(
        lower=0, upper=num_steps, body_fun=body_fn, init_val=init_val)
    p_0_samples.append(z)

  p_1_samples = jnp.concatenate(p_1_samples, axis=0)
  p_0_samples = jnp.concatenate(p_0_samples, axis=0)
  return p_1_samples, p_0_samples
\end{lstlisting}
\end{figure}

\begin{figure}
\centering
\includegraphics[width=0.95\textwidth]{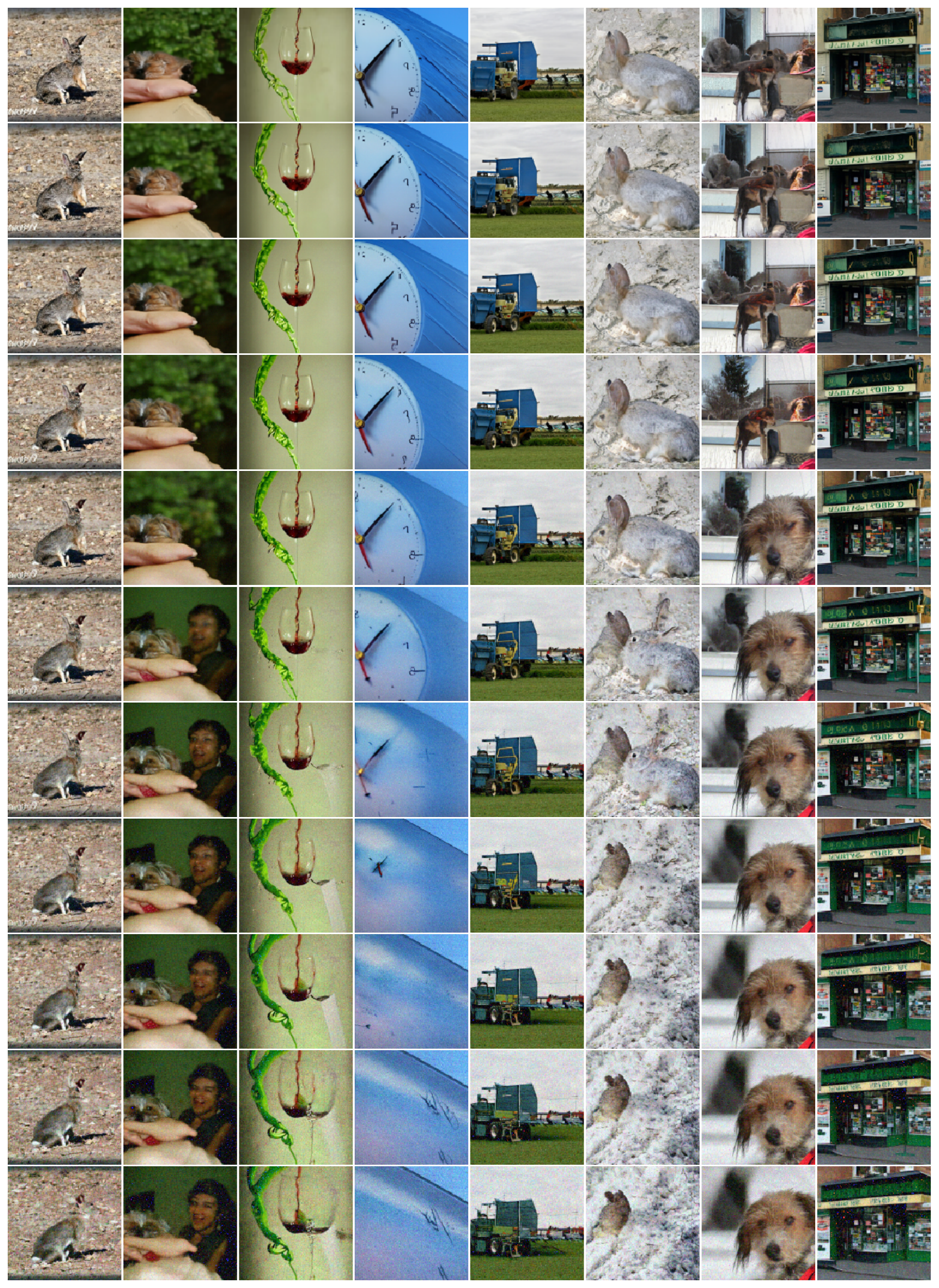}
\caption{%
    \textbf{Constant samples with increasing scaling $\alpha$.}
    Each row displays samples at a particular $g$-scale, from $0$ increasing to $1$ in the increments of $0.1$ from top to bottom.
    Sampling for each columns starts off with the same initial noise image and conditioning class.
}
\label{fig:constantg_scale}
\end{figure}

\begin{figure}
\centering
\includegraphics[width=0.95\textwidth]{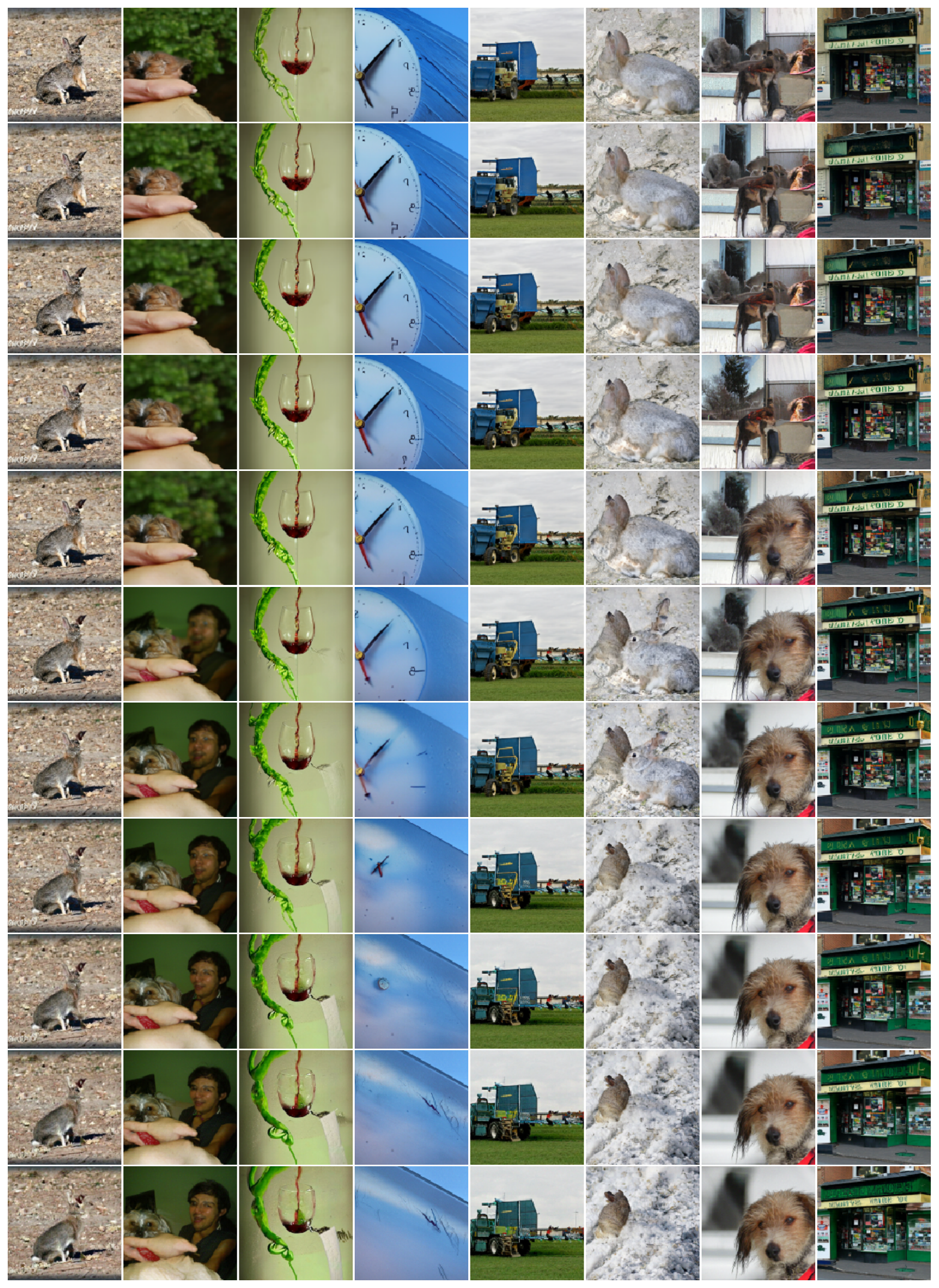}
\caption{%
    \textbf{NonSingular samples with increasing scaling $\alpha$.}
    Each row displays samples at a particular $g$-scale, from $0$ increasing to $1$ in the increments of $0.1$ from top to bottom.
    Sampling for each columns starts off with the same initial noise image and conditioning class.
}
\label{fig:nonsingular_scale}
\end{figure}

\begin{figure}
\centering
\includegraphics[width=0.95\textwidth]{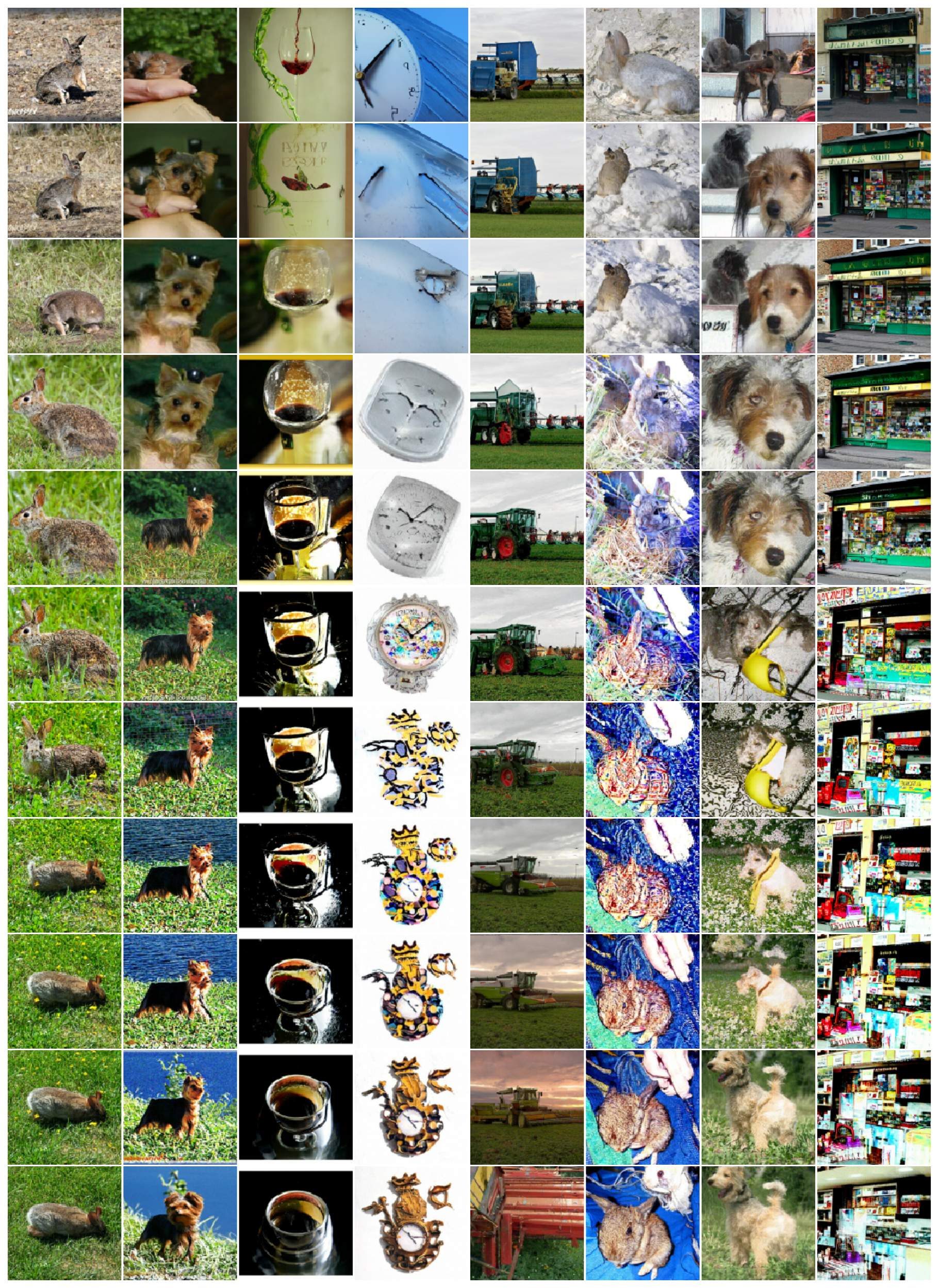}
\caption{%
    \textbf{Singular samples with increasing scaling $\alpha$.}
    Each row displays samples at a particular $g$-scale, from $0$ increasing to $1$ in the increments of $0.1$ from top to bottom.
    Sampling for each columns starts off with the same initial noise image and conditioning class.
}
\label{fig:singular_scale}
\end{figure}

\begin{figure}
\centering
\includegraphics[width=0.99\textwidth]{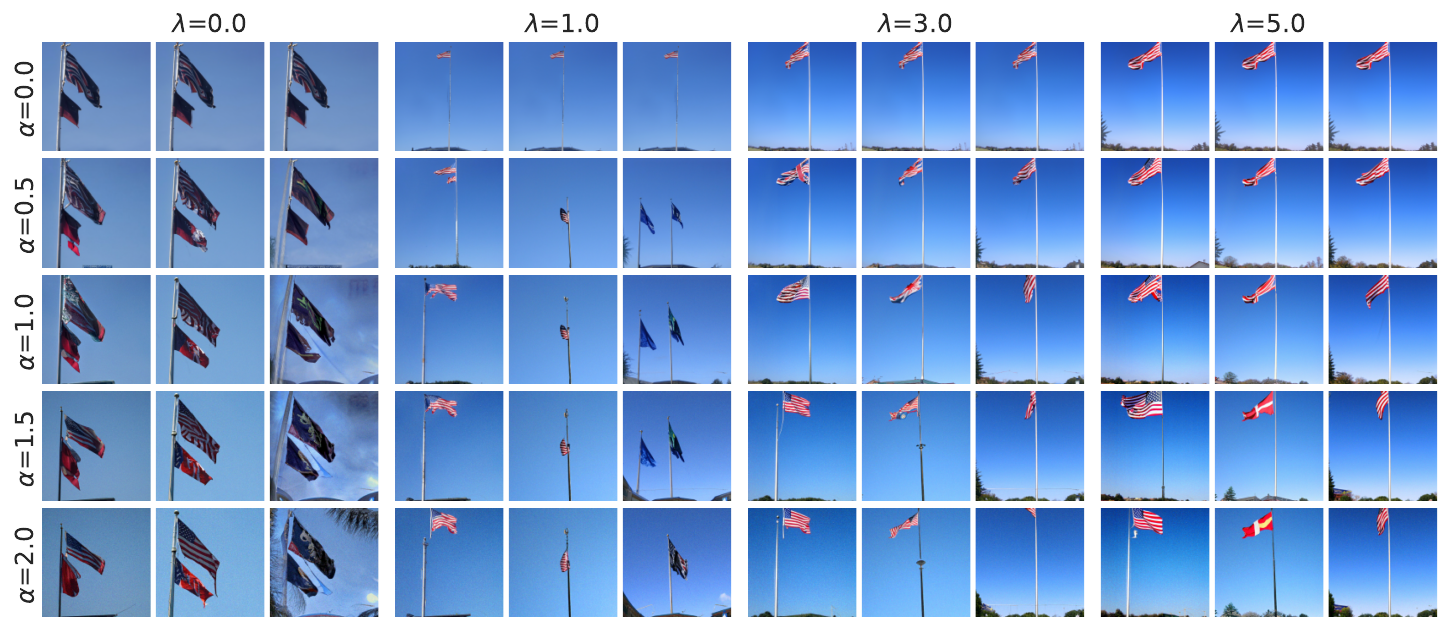}
\includegraphics[width=0.99\textwidth]{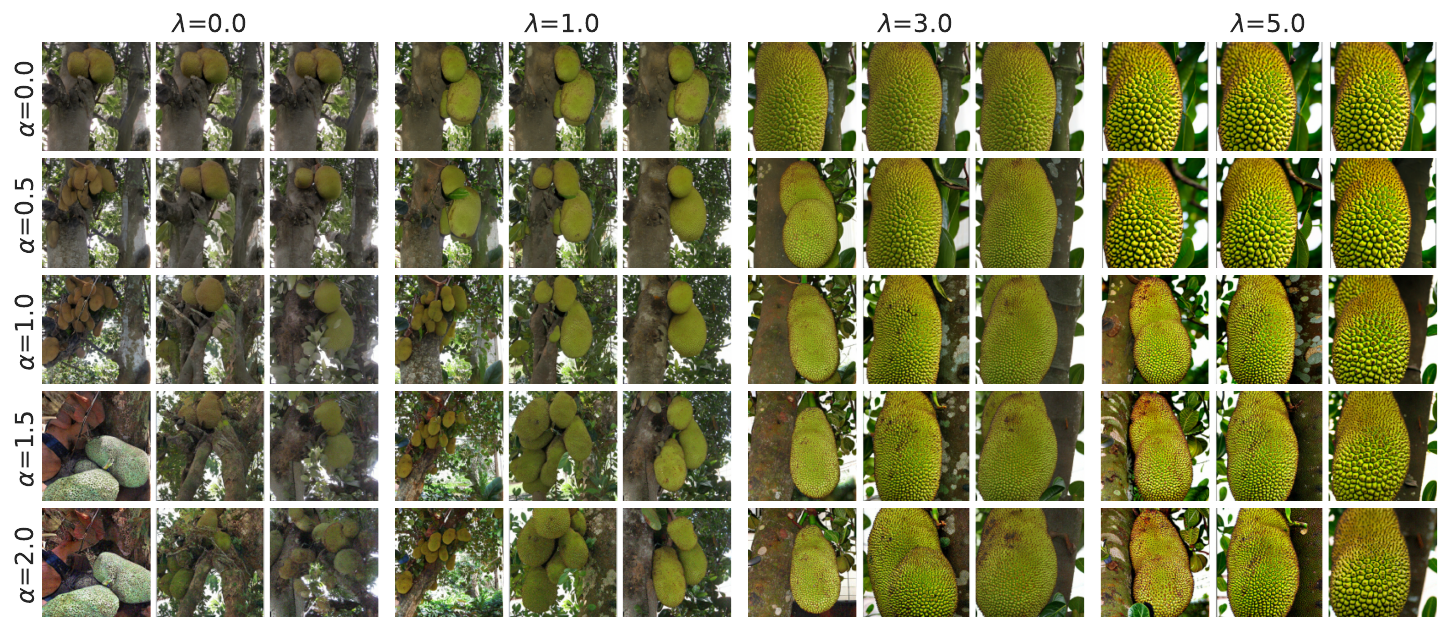}
\includegraphics[width=0.99\textwidth]{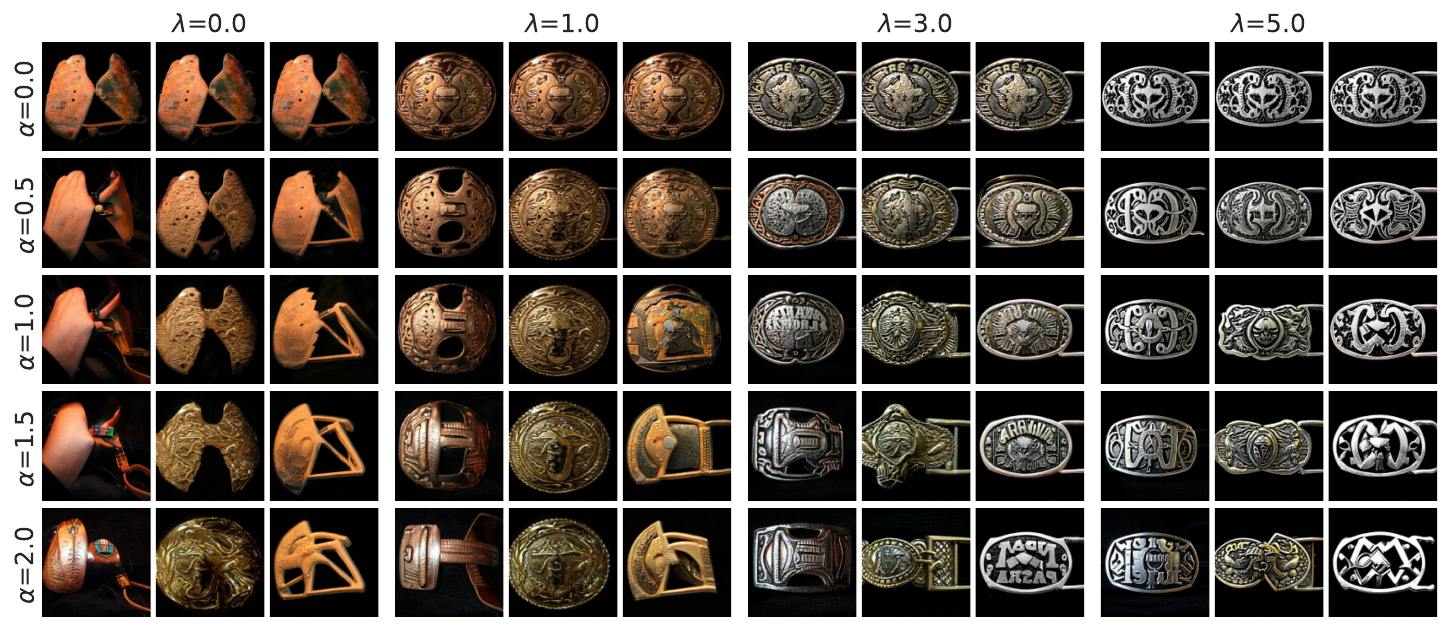}
\caption{%
    \textbf{Stochastic sampling improves diversity at all classifier-free guidance levels.}
    Additional results as in \cref{fig:imagenet_cfg}, described in \cref{sec:imagenet_experiments}.
}
\label{fig:qual_alpha_vs_lambda}
\end{figure}

\end{document}